\documentclass[english, hidelinks, 12pt]{article}
\usepackage{geometry}
\usepackage{setspace}
\usepackage{bm}
\usepackage{graphicx}
\graphicspath{ {.} }
\usepackage{hyperref}
\usepackage{xr}

\usepackage{lineno}

\usepackage{amssymb, amsmath, amsthm, amsfonts, mathtools, bbm}
\usepackage{float}
\usepackage{color, xcolor,colortbl}
\usepackage[dvipsnames]{xcolor}
\usepackage{url}
\usepackage{enumitem}
\usepackage{csquotes}
\usepackage{multirow}
\usepackage{siunitx}
\usepackage{placeins}
\usepackage{authblk}
\usepackage{appendix}
\usepackage{lscape}
\usepackage{booktabs}
\usepackage{longtable}
\usepackage{stackrel}

\usepackage{comment}

\allowdisplaybreaks[1]

\usepackage{tikz}
\usetikzlibrary{arrows.meta, decorations.pathreplacing}

\usepackage{algorithm}
\usepackage{algpseudocode}

\usepackage{xr-hyper}
\usepackage{hyperref}
\hypersetup{
    colorlinks=true,
    linkcolor=blue,
    citecolor=blue,
    filecolor=magenta,
    urlcolor=cyan
}
\usepackage{cleveref}

\usepackage[round]{natbib}

\theoremstyle{plain}
\newtheorem{theorem}{Theorem}
\newtheorem{proposition}{Proposition}
\newtheorem{lemma}{Lemma}
\newtheorem{corollary}[theorem]{Corollary}
\newtheorem{assumption}{Assumption}
\newtheorem{remark}{Remark}

\newtheorem{definition}{Definition}

\numberwithin{equation}{section}

\usepackage{soul}       
\usepackage{ulem}   
\def\1{\mathbf{1}}
\def\0{\mathbf{0}}

\usepackage[makeroom]{cancel}

\usepackage{parskip}

\begin{document}
\def\spacingset#1{\renewcommand{\baselinestretch}%
{#1}\small\normalsize} \spacingset{1}
	
	\begin{titlepage}
		
		\title{{\Large\textbf{
        Perturbation is All You Need for Extrapolating Language Models
        }}}
        
                \author{
                Zetai Cen\thanks{Equal contribution. 
                School of Mathematics, University of Bristol.
                Email: zetai.cen@bristol.ac.uk.
                Supported by Engineering and Physical Sciences Research Council (EP/Z531327/1).},
                \hspace{0pt}
                Jin Zhu\thanks{Equal contribution.
                School of Mathematics, University of Birmingham.
                Email: j.zhu.7@bham.ac.uk},
                \hspace{0pt}
                Xinwei Shen\thanks{Department of Statistics, University of Washington.
                Email: xwshen@uw.edu},
                \hspace{0pt}
                and
                Chengchun Shi\thanks{Corresponding author.
                Department of Statistics, London School of Economics and Political Science.
                Email: c.shi7@lse.ac.uk}
                }

		\date{\empty}
		
		\maketitle

\begin{abstract}
This paper develops a statistical theory of extrapolation for large language models, by reinterpreting them through pre-post-additive noise models.
In contrast to the standard autoregressive next-token prediction based on an exact prefix, we introduce a perturbation-based procedure that first transforms the prefix into a semantic neighbour and then conditions on this perturbed variant for next-token prediction. 
This yields a hierarchical model with a pre-post-additive noise structure.
Within this framework, we develop a rigorous theory of extrapolability, namely, the capacity of a model class to make reliable predictions for token sequences that lie outside the empirical support of the training corpus, by establishing five properties of the proposed procedure: adaptivity, contractivity, robustness, extrapolability, and double robustness.
We evaluate the finite sample performance of the proposed procedure using both synthetic and real world language data. Results show that the proposed method consistently improves out-of-support prediction while maintaining competitive in-support performance, demonstrating that perturbation offers a practical route to language modelling.
\end{abstract}
		

		\noindent
		{\sl Keywords and phrases:}
        Extrapolation,
        large language models,
        perturbation,
        pre-post-additive noise model.

\noindent

	\end{titlepage}
	
	\setcounter{page}{2}

\maketitle


\renewcommand{\baselinestretch}{1.8}

\newpage

\section{Introduction}\label{sec: intro}

\subsection{Overview}

The past few years have witnessed the emergence of highly intelligent large language models (LLMs), including the GPT, Claude, Gemini, Qwen and DeepSeek series. These models are evolving over time and their impact is everywhere: they have been deeply integrated into our daily lives and work, representing an important milestone toward the realization of artificial general intelligence. 
Much of this progress has been driven by scaling: training increasingly large models on increasingly large corpora. However, the scaling paradigm appears to be approaching a hard ceiling. As the availability of high-quality text data appears to reach its limit \citep{villalobos2024position}, the scaling law which links model performance to parameter number and data volume \citep{kaplan2020scaling} suggests that performance gains from continued scaling are diminishing. This calls for more sample-efficient algorithms for further enhancement.

In this paper, we propose a simple yet effective technique to bolster the performance of existing LLMs, supported by a rigorous theoretical framework that characterizes its benefits. Our methodology builds upon autoregressive language models \citep{Bengioetal2000}, which decompose the joint probability mass function of a text sequence into a product of conditional probabilities, each corresponding to the distribution of the next token given the preceding context. Under this formulation, language generation can be viewed as a sequential next token prediction problem: at each time step, the model estimates the conditional distribution of the next token given the current input prefix, randomly draws a token from this distribution, appends it to the prefix, and repeats the procedure to generate the remaining sequence. 

We introduce a two-step procedure for next token prediction: 
(i) we generate a semantic perturbation (e.g., a synonym) of the input prefix; 
and (ii) the model then performs next token prediction conditioned on this perturbed variant rather than the original input. 
Figure~\ref{fig: pretrain_illustration} provides a graphical illustration for the proposed perturbed autoregressive model.

\begin{figure}[t!]
    \centering
    \includegraphics[width=\hsize]{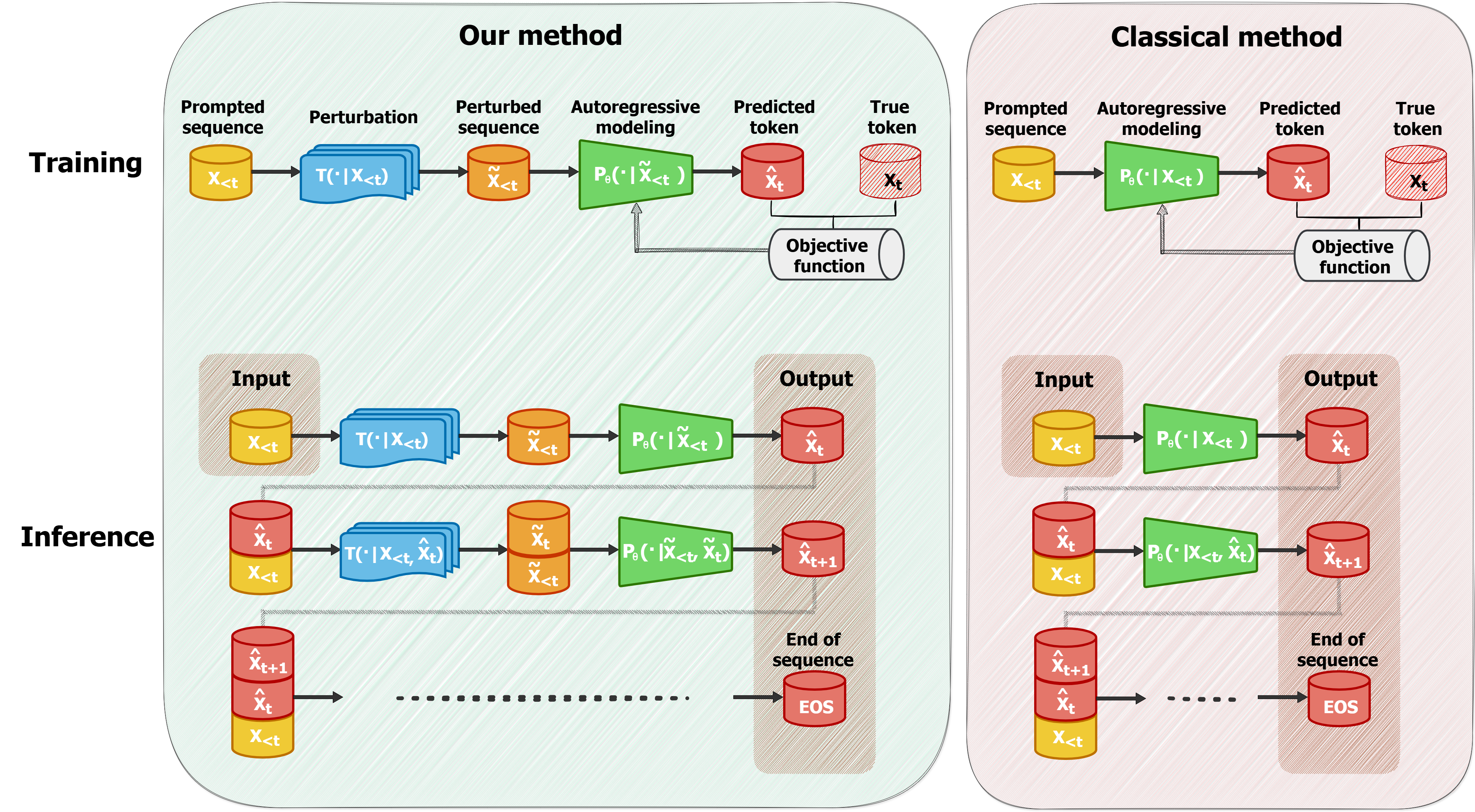}
    \caption{An illustration of our approach, in comparison with the classical LLM training.
    At both training and inference stages, where inference refers to test-time language generation using a trained model, our method processes an input prefix $\bm{X}_{<t}$, consisting of the first $t-1$ tokens, in two steps: (i) drawing a perturbed prefix  $\widetilde{\bm{X}}_{<t}$ from a perturbation distribution $T(\cdot| \bm{X}_{<t})$ and (ii) predicting the next token $\widehat{X}_t$ via the autoregressive model $P_\theta(\cdot| \widetilde{\bm{X}}_{<t})$.
    The classical method consists of our step~(ii) only, such that $\widehat{X}_t$ is drawn directly from $P_\theta(\cdot| \bm{X}_{<t})$.
    }
    \label{fig: pretrain_illustration}
\end{figure}

\subsection{Motivation}

Our use of perturbations for language modelling arises naturally from research in linguistics and cognitive science, which suggests that human language comprehension and generation are fundamentally robust to textual variation. Sophisticated cognitive processes frequently rely on analogical reasoning and metaphor, often remaining unaffected by semantic substitutions \citep{EricksonMattson1981, Satoetal2015}. As discussed by \citet{Satoetal2015}, people habitually speak about abstract concepts metaphorically.
For instance, describing a person as being ``full of love'' is conceptually equivalent to saying they have ``a lot of love to give''. This flexibility is further supported by neurocognitive evidence. For example, \citet{Lietal2024} demonstrate that reasoning on syntactic patterns involves brain regions specifically dedicated to semantic analogies. Collectively, these findings indicate that human language production is inherently resilient to and even driven by textual perturbations.

We bring this enlightenment to LLMs, by explicitly modelling the metaphor or analogy as a probabilistic perturbation process. 
In this paper, we adopt simple techniques (e.g., random insertion, deletion, and swap, see Section~\ref{sec: method}) to design the perturber, so as to pinpoint the importance of including perturbation per se.
We highlight that the above analogy motivates the perturbation mechanism, but our formal results do not rely on cognitive assumptions.

As a more important point of view,
we argue that employing such a perturbed model for pretraining is reasonable. Not only does it mimic the cognition of human beings, but it also addresses a crucial yet seldom studied problem in language generation: extrapolation. In our specific context, extrapolation can be thought of as the capability of LLMs to generate reasonable answers to unseen contexts during training. In practical applications, a model almost never encounters the exact prompts or contexts seen during training; inputs are inherently out-of-support or may even deviate substantially from the training corpus. Furthermore, we posit that the capacity for extrapolation is intrinsically linked to a model's creativity and the emergence of complex reasoning, namely, capabilities that require the model to generalize beyond rote memorization. 
Consequently, it is expected that an LLM specifically trained to enhance its extrapolative power would outperform standard models.

\subsection{Related works}\label{subsec: literature}

In machine learning, extrapolation is concerned with the performance of a trained model on data points that lie outside its training support. Recent works have advanced our understanding of the extrapolability of machine learning algorithms \citep{Christiansenetal2022, DongMa2023, Kongetal2024}.
Our proposal is related to \textit{engression} \citep{ShenMeinshausen2025}, which
considers the standard regression setting with a scalar response $Y$ and a vector of predictors $X$. Its goal is to learn the conditional distribution function of $Y$ given $X$. Toward that end, they distinguish among three types of regression models, depending on whether the noise enters after the regression function, before the regression function through perturbation of the covariates, or both:
\begin{enumerate}
    \item [(i)] \textit{Post-additive noise model}: $Y=f(X)+\varepsilon$;
    \item [(ii)] \textit{Pre-additive noise model}: $Y=f(X+\eta)$;
    \item [(iii)] \textit{Pre-post-additive noise model}: $Y=f(X+\eta)+\varepsilon$.
\end{enumerate}
In all three models, $f$ denotes a nonlinear deterministic function to be estimated, whereas 
$\varepsilon$ and $\eta$ denote some random noises. Most regression models studied in the literature fall into the first type, in which the noise is added to the response and its effect is independent of the regression function $f$. The second type, the pre-additive noise model, has the implication that each response $Y$ is generated not directly from the corresponding predictor 
$X$, but from a perturbed version $X+\eta$. The third type combines the first two.
Engression advocates the second type of model over the classical first type for effective distributional learning. This preference is motivated by the theoretical finding that pre-additive noise models can extrapolate under certain conditions.
In contrast,
we primarily utilize the third type 
by introducing a new model class whose extrapolation properties are then characterized using a related notion of uncertainty.

LLMs aim to characterize the probability distribution over text. 
Learning such models is one of the most fundamental problems in natural language processing.
In modern practice, these models are predominantly trained using Transformer-based autoregressive architectures \citep{vaswani2017attention}.
Pretraining refers to the procedure that encodes general linguistic knowledge into these high-capacity language models using massive human text corpora. This is typically achieved via maximum likelihood estimation (MLE), which is equivalent to minimizing the Kullback--Leibler (KL) divergence between the model’s predictive distribution and the empirical distribution of human texts \citep{radford2018improving}. More recently, \citet{ji2023tailoring}, \citet{zhang2023mixce}, and \citet{ren2024emo} proposed to minimize alternative divergence measures, such as the total variation distance and Earth Mover’s distance,
to mitigate the sensitivity of KL divergence to small perturbations in the data. 
Different from these works, which propose alternative training objectives, we propose an alternative model architecture based on perturbation. Our model can nonetheless be trained under these objectives.

\subsection{Contribution and organization}

To summarize, we make the following contributions. 
Methodologically, we propose a perturbed autoregressive language model, that differs from classical LLMs whose training relies exclusively on unperturbed sequences. 
Our approach also deviates from standard data augmentation that utilizes perturbations solely for enriching the training dataset.

Theoretically, we establish the following statistical properties to characterize how perturbation affects extrapolation across three regimes of the base model’s capability (\textit{perfect}, \textit{moderate}, and \textit{poor}); see Figure~\ref{fig:theory} for a visualization of our theoretical results. (i) First,  \ul{\textit{adaptivity}} shows that if the base model class is perfectly extrapolable, then the perturbed class inherits this property (Theorem~\ref{thm: simple_extrap_feature_map}). (ii) Second, \ul{\textit{contractivity}} shows that when the base class has imperfect extrapolation capability, perturbation can further improve it (Theorem~\ref{thm: contraction_global_extrap}). (iii) Third, \ul{\textit{robustness}} provides meaningful extrapolation guarantees even when the base class extrapolates poorly (Theorem~\ref{thm: just_perturb}). (iv) Fourth, \ul{\textit{extrapolability}} establishes that, even in the latter two regimes, the perturbed model class can still extrapolate perfectly with a suitably chosen perturber (Theorem~\ref{thm: extrap_class_charac_perturb}). (v) Finally, combining adaptivity and extrapolability yields a \ul{\textit{double robustness}} property: the perturbed model class is perfectly extrapolable if either the base model class is extrapolable or the perturber satisfies the required condition (Corollary \ref{coro:dr}).

\begin{figure}[t]
    \centering
    \includegraphics[width=0.7\linewidth]{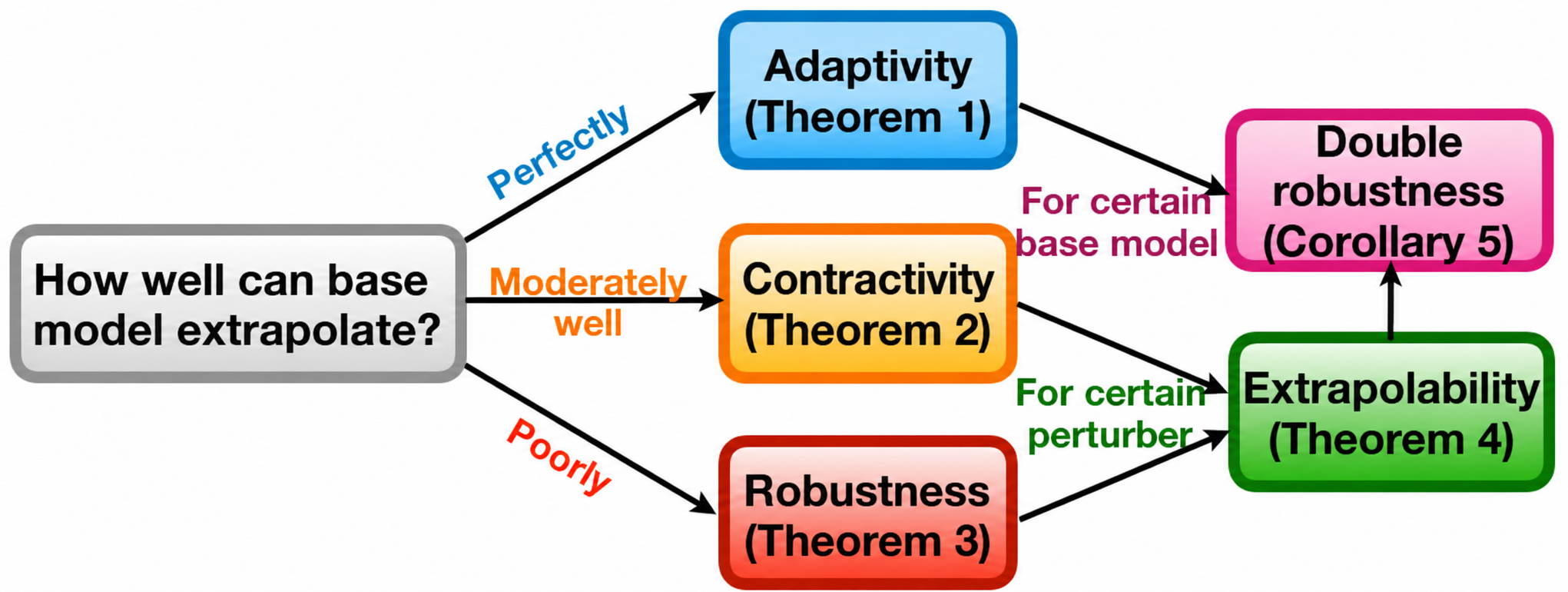}
    \caption{Roadmap of our theoretical results.}
    \label{fig:theory}
\end{figure}

Importantly, these theories are substantially different from those in \citet{ShenMeinshausen2025},  which consider tabular regression and study how pre-additive-noise models improve upon post-additive-noise models. In contrast, we consider language models with a different pre-post-additive structure, which yields adaptivity, robustness, and extrapolability guarantees that are absent from their analysis. 
We also note that the existing literature on LLMs is largely heuristic. By providing rigorous theoretical guarantees, we aim to illustrate how statistical theory can offer a principled foundation for understanding and improving LLMs.

The rest of this paper is organized as follows.
We introduce the perturbed language model and detail its parameter estimation in Section~\ref{sec: method}.
Then we establish in Section~\ref{sec: theory} a rigorous theoretical framework to study the extrapolation property of the proposed language model. 
We next investigate numerically how well the proposed method extrapolates, using synthetic experiments in Section~\ref{sec:simulation} and real-world applications in Section~\ref{sec:real-world}, respectively. Section~\ref{sec: conclusion} concludes the paper.
All the detailed proofs, additional numerical results and discussions are included in the Supplementary Materials.

\section{Methodology}\label{sec: method}

\subsection{Models and training}

We first formulate the pretraining problem in language modelling.
Let $\bm{X}$ denote a human-authored sentence,  represented as a sequence of tokens $(X_1, \ldots, X_L)$, where each token $X_t$ (encoded as an integer) is a word, subword, or punctuation symbol drawn from a vocabulary, denoted as a finite set $\mathcal{V}$. We assume a fixed sequence length $L$, where shorter sentences are extended to this length via zero padding, i.e., by appending a sequence of zeros until they reach length $L$. Let $P^*$ denote the probability distribution over $\bm{X}$. Our goal is to learn a parametric distribution $P_{\theta}$, parametrized by $\theta$, that is as close to the human language model $P^*$ as possible.

For any $t>1$, define $\bm{X}_{<t}:= (X_1,\dots,X_{t-1})$ as the sequence of tokens preceding the $t$th token, with the convention that $\bm{X}_{<1}=\emptyset$. By the chain rule, both $P^*$ and 
$P_\theta$ can be factorized into products of conditional distributions such that
\begin{align*}
    P^*(\bm{X}) &= \prod_{t=1}^L P^*(X_t | \bm{X}_{<t}) , \quad \\
    P_{\theta}(\bm{X}) &= \prod_{t=1}^L P_{\theta}(X_t | \bm{X}_{<t}) .
\end{align*}
Autoregressive language models are motivated by this factorization. Given a prefix $\bm{X}_{<t}$ of length $t-1$, an autoregressive language model outputs a probability distribution $P_{\theta}(\cdot|\bm{X}_{<t})$ over the vocabulary, conditional on the preceding tokens. The next token is then sampled from this distribution, and this procedure repeats for $t=1,2,\dots$, until an end-of-sequence (EOS) token is produced, at which point the sentence is complete. Note that the stochasticity in such classical models arises solely from the sampling of the next-token distribution. Consequently, these models are conceptually similar to post-additive noise models in classical regression (Section \ref{subsec: literature}).

Our approach differs from existing LLMs,
with the intuition inspired by the favourable extrapolation properties of pre- and pre-post-additive noise models proposed in the standard regression setting. As discussed in Section~\ref{subsec: literature}, these models do not link the predictor directly to the response; instead, the response depends on a perturbed version of the predictor that resides in its neighbourhood. Such a structure inherently captures out-of-support information: even when test data lies outside the training support, the model can still extrapolate effectively provided that the test data can be viewed as ``neighbours'' to the training data through these perturbations. We extend this mindset to language modelling. 

Specifically, we propose a novel \textit{perturbed autoregressive model} that can be applied on top of \textit{any} existing autoregressive LLM. Given a sequence of preceding tokens $\bm{X}_{<t}$, we first randomly perturb it to obtain its contextual neighbour, denoted by $\widetilde{\bm{X}}_{<t}$, which need not have the same length as $\bm{X}_{<t}$. This perturbed sequence is generated according to
\[
\widetilde{\bm{X}}_{<t} \sim T(\cdot | \bm{X}_{<t}) ,
\]
where $T$ denotes a specific perturbation distribution to be detailed later. We next feed $\widetilde{\bm{X}}_{<t}$ into a base autoregressive language model $P_\theta(\cdot| \widetilde{\bm{X}}_{<t})$, parametrized by $\theta$, to predict the next token, and repeat this procedure for $t = 1, 2, \ldots$, until the EOS token is generated.
It is worthwhile to mention that generating perturbation from $T(\cdot |\bm{X}_{<t})$ is conceptually similar to injecting pre-additive noise as in $X+\eta$ in classical regression. Meanwhile, the stochastic sampling from $P_\theta(\cdot| \widetilde{\bm{X}}_{<t})$ injects post-additive noise during next-token generation. The perturbed model can thus be viewed as an adaption of the pre-post-additive noise model described in Section~\ref{subsec: literature} to LLMs.

To train the model,
given a prefix $\bm{X}_{<t}$ from a training sequence $\bm{X}$, MLE boils down to maximizing
\begin{align*}
    \log\big[ \mathbb{E}_{\widetilde{\bm{X}}_{<t} \sim T(\cdot|\bm{X}_{<t})} \big\{ P_{\theta}(X_t|\widetilde{\bm{X}}_{<t}) \big\} \big] ,
\end{align*}
which is similar to the log-likelihood of a hierarchical model.
For computational efficiency, we instead maximize a lower bound, $\mathbb{E}_{\widetilde{\bm{X}}_{<t} \sim T(\cdot|\bm{X}_{<t})} \big\{ \log\big[ P_{\theta}(X_t|\widetilde{\bm{X}}_{<t}) \big] \big\}$, by Jensen's inequality.
To this end, we first sample $m$ number of perturbed sequences $\{\widetilde{\bm{X}}_{<t}^{(i)}\}_{i=1}^m$ such that 
\[
\widetilde{\bm{X}}_{<t}^{(i)} \sim T(\cdot| \bm{X}_{<t}) ,
\quad
i=1,\dots,m .
\]
By aggregating these over all $t$, this gives rise to our objective function for a single sequence, 
\begin{equation}
\label{eqn: obj_log}
\begin{split}
    &\mathcal{L}^{\text{Log}}(P_\theta; \bm{X}) :=\frac{1}{m}
    \sum_{i=1}^m\sum_{t=1}^L\log P_\theta(X_t | \widetilde{\bm{X}}_{<t}^{(i)}) .
\end{split}
\end{equation}
Finally, we aggregate $\mathcal{L}^{\text{Log}}(P_\theta; \bm{X})$ over all sequences $\bm{X}$ in a training corpus $\mathcal{D}$. The parameter $\theta$ is estimated by maximizing  $\sum_{\bm{X}\in \mathcal{D}}\mathcal{L}^{\text{Log}}(P_\theta; \bm{X})$.
Note that while multiple perturbed sequences are sampled during training to reduce the randomness arising from the perturbation distribution $T$, we employ a single perturbation during inference\footnote{Here, inference is used in the machine-learning sense: it refers to the deployment stage in which a trained model is used for language generation, and should be distinguished from the conventional use of inference in statistics, where it refers to uncertainty quantification.} for the sake of computational efficiency.
We summarize the training and inference procedures of our proposed model in Algorithms~\ref{alg: pretrain} and \ref{alg: inference}, respectively.

\begin{algorithm}[h!t!]
\caption{Perturbed autoregressive training.}
\label{alg: pretrain}
\begin{algorithmic}[1]
\setstretch{1.25}
\State \textbf{Input:} training corpus $\mathcal{D} = \{ \bm{X}^{(i)} \}_{i=1}^n$, $m$ perturbation steps, random insertion $T$ 
\For{$i=1,\dots, n$}
\For{$t=1,\dots, |\bm{X}^{(i)}|$}
\For{$j=1,\dots, m$}
\State Sample $\widetilde{\bm{X}}_{<t}^{(i,j)} \sim T(\cdot| \bm{X}_{<t}^{(i)})$
\EndFor
\EndFor
\EndFor
\State Set $\mathcal{L}(\theta)= \sum_{i=1}^n \sum_{t=1}^{|\bm{X}^{(i)}|} \sum_{j=1}^m \log P_\theta(X_t^{(i)}| \widetilde{\bm{X}}_{<t}^{(i,j)})$
\State \textbf{Output:} $\arg\max_\theta \mathcal{L}(\theta)$
\end{algorithmic}
\end{algorithm}

\begin{algorithm}[h!t!]
\caption{Perturbed autoregressive inference.}
\label{alg: inference}
\begin{algorithmic}[1]
\setstretch{1.25}
\State \textbf{Input:} prompt $\bm{X}_0$, random insertion $T$, parameter $\theta$
\State Set $\bm{X}_{<1} \gets \bm{X}_0$, $t\gets 0$
\Repeat
\State $t\gets t+1$
\State Sample $\widetilde{\bm{X}}_{<t} \sim T(\cdot| \bm{X}_{<t})$ 
\State Sample $X_t \sim P_\theta(\cdot | \widetilde{\bm{X}}_{<t})$ 
\State $\bm{X}_{<t+1} \gets (\bm{X}_{<t}, X_t)$ 
\Until $X_t$ is EOS
\State \textbf{Output:} $\bm{X}_{<t}$
\end{algorithmic}
\end{algorithm}

\subsection{Perturbation design}

Our proposal is related to data augmentation in that we can apply techniques such as synonym replacement and random insertion or deletion to construct perturbations.
Early approaches in natural language processing augment the data by replacing synonyms or keywords, inserting and deleting words or characters in the training text. Despite their simplicity, these techniques have been empirically shown to improve the model performance \citep{wei2019eda, feng2020genaug}. 
To be simple and computationally effective, we primarily utilize random insertion in our implementation,
which proceeds as follows:
\begin{itemize}
    \item [(i)] A subset of tokens is randomly selected from the prefix;
    \item [(ii)] Synonyms for these tokens are generated;
    \item [(iii)] These synonyms are inserted at random positions within the prefix (prior to the EOS token). 
\end{itemize}
Thus, the perturbation distribution $T$ is embedded through the transformation that takes a prefix as input and produces an output through steps (i)--(iii).
Synonym replacement and random deletion are also applicable, as shown in additional experiments in the Supplementary Materials.
It is worth pointing out that training and using our proposed models are as computationally efficient as any adopted base models. The additional computation takes at most a few minutes, compared with at least several hours required to train a base LLM.

To conclude this section, we remark a technical nuance which distinguishes our proposal from data augmentation.
The difference lies in that we introduce perturbations not only during training but also at inference time (e.g., when generating the model output), thus changing how the conditional distribution is modeled;
whereas classical augmentation methods apply perturbations exclusively during training and hence only tweak the empirical data distribution used for estimation.
Hence in our proposal, perturbation is an intrinsic component of the LLM architecture, rather than a technique to enrich the training dataset. 
We will elucidate in Section~\ref{sec: theory} why such a simple change could lead to nontrivial improvement,
and further demonstrate the advantages of perturbing at both training and inference time in our ablation study in Section~\ref{sec:ablation}.

\section{Theoretical analysis}\label{sec: theory}
Before detailing our theories, 
we provide a high-level summary of our theoretical findings. We establish the following five theoretical properties on the extrapolability of the proposed perturbed autoregressive model (see also, Figure \ref{fig:theory}):
\begin{itemize}[leftmargin=*]
    \item \textbf{Adaptivity} (Theorem~\ref{thm: simple_extrap_feature_map}): If the base class is already extrapolable, then the perturbed class inherits this property.
    \item \textbf{Contractivity} (Theorem~\ref{thm: contraction_global_extrap}): Perturbation can further improve the extrapolation capability of the base model class, rather than merely preserving it. 
    \item \textbf{Robustness} (Theorem~\ref{thm: just_perturb}): Even when the base model class has limited extrapolation capability, the perturbed model class continues to enjoy meaningful extrapolation guarantees that depend only on the perturber.
    \item \textbf{Extrapolability} (Theorem~\ref{thm: extrap_class_charac_perturb}): With a suitably chosen perturber, the perturbed model class can achieve perfect extrapolability.
    \item \textbf{Double robustness} (Corollary \ref{coro:dr}): The perturbed model class is extrapolable if either the base model class or the perturber is suitably chosen.
\end{itemize}

\subsection{Definition and basic properties}\label{subsec: def_extrap}
To formally characterize extrapolability, we first introduce some notations. Let $\mathcal{X}^*$ denote the space of all possible prefixes (i.e., 
the set of all token sequences of length $L$
over vocabulary $\mathcal{V}$), 
and let $\mathcal{X} \subseteq \mathcal{X}^*$ represent the training support, consisting of the subset of sequences observed during pretraining. 
We equip $\mathcal{X}^*$ with a metric $M(\cdot, \cdot)$ that measures the semantic distance between any two prefixes, such as the Hamming distance.
Let the distance between a prefix $\bm{X}$ and the training support $\mathcal{X}$ be $d(\bm{X}, \mathcal{X}) := \inf_{\bm{X}' \in \mathcal{X}} M(\bm{X}, \bm{X}')$.
For a given prefix $\bm{X}$, we further define the total variation distance between two token prediction distributions $P_\theta(\cdot|\bm{X})$ and $P_{\theta'}(\cdot|\bm{X})$ as $\mathbb{TV}_{\bm{X}}(P_{\theta}, P_{\theta'}) := \mathbb{TV}[ P_{\theta}(\cdot | \bm{X}), P_{\theta'}(\cdot | \bm{X}) ]$.

In contrast to classical statistical learning theory that primarily studies interpolation (viz.\ within the support of the training distribution), we study distributional extrapolability through the following notion of extrapolation uncertainty, adapted from \citet{ShenMeinshausen2025}, on a language model class $\mathcal{P}=\{P_{\theta}\}_{\theta
}$.

\begin{definition}[Extrapolation uncertainty and extrapolability]
\label{def: dist_extrap}
For a given $\delta \geq 0$, the 
\textit{extrapolation uncertainty} of a class of LLMs $\mathcal{P}$ 
is defined as
\begin{align*}
    &
    \mathcal{U}_{\mathcal{P}}(\delta) 
    := 
    \hspace{-3pt}
    \sup_{
    \substack{
    \text{\footnotesize $\bm{X}':$ } \\
    \text{\footnotesize $d(\bm{X}', \mathcal{X}) \leq \delta$ }
    }
    } 
    \sup_{
    \substack{
    \text{\footnotesize $P_{\theta}, P_{\theta'} \in \mathcal{P}:$ } \\
    \text{\footnotesize $\mathbb{TV}_{\bm{X}}(P_{\theta}, P_{\theta'}) =0 $ } \\
    \text{\footnotesize $ \forall \bm{X} \in \mathcal{X}$ }
    }
    }
    \hspace{-3pt}
    \mathbb{TV}_{\bm{X}'}(P_{\theta}, P_{\theta'}) .
\end{align*}
We say $\mathcal{P}$ is \textit{extrapolable}
if $\mathcal{U}_{\mathcal{P}}(\delta) = 0$ for any $\delta \geq 0$.
\end{definition}

In words, the extrapolation uncertainty $\mathcal{U}_{\mathcal{P}}(\delta)$ for a model class $\mathcal{P}$ can be understood as the worst performance of its models that are perfectly trained on $\mathcal{X}$, over any testing sequence $\bm{X}'$ that is at most $\delta$ away from the training support (quantified by the distance $d(\cdot, \mathcal{X})$). 
Intuitively, it represents the maximal disagreement, characterized by the total variation distance, among all models in $\mathcal{P}$ that are observationally indistinguishable on the training support, i.e., $\mathbb{TV}_{\bm{X}}(\cdot, \cdot) =0$ for all $\bm{X} \in \mathcal{X}$.
If these models behave substantially differently at a nearby point $\bm{X}' \notin \mathcal{X}$, prediction based on data in $\mathcal{X}$ would be largely uncertain.
The smaller this uncertainty value, the stronger the model’s extrapolation capability beyond the training support. In the extreme case where $\mathcal{U}_{\mathcal{P}}(\delta) = 0$, the model class is considered perfectly extrapolable.

In what follows, we present several basic properties on Definition~\ref{def: dist_extrap}.

\begin{proposition}[Basic properties of extrapolation uncertainty]\label{prop: basic_extrap}
We have:
\begin{itemize}
    \item [(i)] (Interpolation case). 
    If $\delta =0$, then $\mathcal{U}_{\mathcal{P}}(\delta) =0$ for any model class $\mathcal{P}$.
    \item [(ii)] (Monotonicity). 
    For any model class $\mathcal{P}$ and any $0\leq \delta_1\leq \delta_2$, we have
    \begin{align*}
        \mathcal{U}_{\mathcal{P}}(\delta_1) \leq
    \mathcal{U}_{\mathcal{P}}(\delta_2).
    \end{align*}
    \item [(iii)] (Right continuity).
    For any model class $\mathcal{P}$, $\mathcal{U}_{\mathcal{P}}(\delta)$ is right-continuous.
    \item [(iv)] (Well-definedness).
    The quantity $\mathcal{U}_{\mathcal{P}}^{(\infty)} := \lim_{\delta \to \infty} \mathcal{U}_{\mathcal{P}}(\delta)$ exists.
\end{itemize}
\end{proposition}
Properties (i)--(iii) are intuitive. Since $\delta$ measures the distance from the training support, the extrapolation uncertainty naturally increases with $\delta$. When $\delta=0$, evaluation occurs on the training support and therefore corresponds to interpolation, in which case the uncertainty vanishes to zero. 
While right continuity is a natural property for a metric defined on spaces of discrete sequences, property~(iii) confirms this.
Lastly,
property~(iv) establishes that $\mathcal{U}_{\mathcal{P}}^{(\infty)}$ is well defined, which serves as the object of Theorem~\ref{thm: contraction_global_extrap} below. 
Hereafter, 
we refer to $\mathcal{U}_{\mathcal{P}}^{(\infty)}$ as the \emph{global extrapolation uncertainty},
and to $\mathcal{U}_{\mathcal{P}}(\delta)$ for a finite $\delta>0$ as the \emph{local extrapolation uncertainty}.

\subsection{Extrapolation guarantees of the perturbed model class}\label{subsec: extrap_guarantee}
We delineate our theoretical results in this section. First, we define two model classes for our analysis. The first is the base model class 
$\mathcal{P}^{\textnormal{base}} = \{ P_\theta (\cdot| \bm{X}) : \bm{X} \in \mathcal{X}^*, \theta \in \Theta^{\textnormal{base}} \}$,  which serves as the foundation for our proposed approach. The second is the proposed perturbed model class (denoted by $\mathcal{P}^{\textnormal{perturb}}$), obtained by first drawing a perturbed prefix $\widetilde{\bm{X}}$ from $T(\cdot | \bm{X})$ and then sampling the next token from the base distribution $P_\theta(\cdot| \widetilde{\bm{X}})$. By the law of total probability, the distribution of each model in $\mathcal{P}^{\text{perturb}}$ can be represented as $\mathbb{E}_{\widetilde{\bm{X}}\sim T(\cdot | \bm{X})} P_\theta(\cdot| \widetilde{\bm{X}})$. Hence,
\begin{align*}
    \mathcal{P}^{\text{perturb}} = \left\{ \mathbb{E}_{\widetilde{\bm{X}}\sim T(\cdot | \bm{X})} P_\theta(\cdot| \widetilde{\bm{X}}) \,: \bm{X} \in \mathcal{X}^* , \theta \in \Theta^{\textnormal{base}} \right\} .
\end{align*}
Our theoretical analysis compares the extrapolation properties of the two model classes across three regimes for the base model: (a) it extrapolates perfectly; (b) it extrapolates, but imperfectly; and (c) it extrapolates poorly. Our first result concerns regime (a), formalized through Definition~\ref{def: dist_extrap}, under which the local extrapolation uncertainty of the base class satisfies $\mathcal{U}_{\mathcal{P}^{\textnormal{base}}}(\delta)=0$ for any $\delta>0$. 
Under the following assumption, we characterize the adaptivity of the perturbed model class.

\begin{assumption}
\label{ass: UI}
For any $P_\gamma, P_{\gamma'} \in \mathcal{P}^{\textnormal{perturb}}$ with $\mathbb{TV}_{\bm{X}}(P_\gamma, P_{\gamma'}) =0$ for all $\bm{X} \in \mathcal{X}$, their corresponding base models $P_\theta, P_{\theta'}$ also satisfy $\mathbb{TV}_{\bm{X}}(P_\theta, P_{\theta'}) =0$ for all $\bm{X} \in \mathcal{X}$.
\end{assumption}

\begin{theorem}[Adaptivity]
\label{thm: simple_extrap_feature_map}
Under Assumption~\ref{ass: UI}, if $\mathcal{P}^{\textnormal{base}}$ is extrapolable (recall Definition~\ref{def: dist_extrap}), 
then $\mathcal{P}^{\textnormal{perturb}}$ is extrapolable.
\end{theorem}

Assumption~\ref{ass: UI} can be viewed as an identifiability condition on the base models. The implication of Theorem~\ref{thm: simple_extrap_feature_map} is that under some identifiability conditions on the perturber, if the base LLM can extrapolate, then the perturbed LLM will preserve this property. While this result may initially seem to imply that perturbation provides no additional benefit, Theorems~\ref{thm: contraction_global_extrap}--\ref{thm: extrap_class_charac_perturb} below establish that perturbation can in fact offer guarantees beyond those available for the base model.

We also remark that even without Assumption~\ref{ass: UI}, the perturbed class $\mathcal{P}^{\textnormal{perturb}}$ may still extrapolate well, subject to an additional perturbation cost that is explicitly defined in Theorem~B.1~(i) of the Supplementary Materials.
In view of this, Theorem~\ref{thm: simple_extrap_feature_map} is a consequence when such perturbation cost is zero, which would hold if the perturbation transformation $T$ is invertible in some sense; we refer to the Supplementary Materials on how Assumption~\ref{ass: UI} can be relaxed.

We next highlight a contraction property.
To this end,
define the Dobrushin coefficient of $T$ as $\alpha(T) := \sup_{\bm{X}, \bm{X}' \in \mathcal{X}^*} \mathbb{TV} ( T(\cdot| \bm{X}), T(\cdot| \bm{X}') )$ \citep{Dobrushin1956}.
We have the following regarding the global extrapolation uncertainty $\mathcal{U}_{ \mathcal{P}}^{(\infty)}$ defined in Proposition~\ref{prop: basic_extrap}~(iii).

\begin{theorem}[Contractivity]
\label{thm: contraction_global_extrap}
Let Assumption~\ref{ass: UI} hold.
For any base model class $\mathcal{P}^{\textnormal{base}}$ and its perturbed class $\mathcal{P}^{\textnormal{perturb}}$ by $T$, 
we have $\mathcal{U}_{ \mathcal{P}^{\textnormal{perturb}}}^{(\infty)} \leq  
2 \alpha(T)\, \mathcal{U}_{ \mathcal{P}^{\textnormal{base}}}^{(\infty)}$.
\end{theorem}

Theorem~\ref{thm: contraction_global_extrap} shows that perturbation acts as a contraction operator on the global extrapolation uncertainty. In particular, whenever $\alpha(T) < 1/2$, the perturbed model class strictly improves the extrapolability of the underlying base class. At the end of this section, we use token replacement as a concrete perturbation strategy and show that it satisfies this contraction condition; see Proposition~\ref{prop: extrap_synonym}.

This theorem is particularly meaningful in regime (b), where the base model extrapolates reasonably well but remains imperfect. In this regime, it shows that the perturbed model class achieves strictly better extrapolation properties under mild conditions. It also yields an immediate adaptivity guarantee, in terms of global extrapolation uncertainty, in regime (a), where the base model extrapolates perfectly. Its upper bound, however, may not be tight in regime (c), where the base model extrapolates poorly. To address this limitation,  
the following theorem demonstrates a critical safety property for the local extrapolation uncertainty:  
the proposed framework remains robust to the failures of the base class in regime (c).

\begin{theorem}[Robustness]
\label{thm: just_perturb}
There exists a constant $C_{T,\delta}$ depending only  on $T$ and $\delta$, independent of $\mathcal{P}^{\textnormal{base}}$, such that
$
\mathcal{U}_{
    \text{\small 
    $\mathcal{P}^{\textnormal{perturb}}$
    }
    }(\delta)
    \leq
    2 C_{T,\delta}$.
\end{theorem}
Theorem~\ref{thm: just_perturb} establishes an upper bound on the local extrapolation uncertainty of $\mathcal{P}^{\textnormal{perturb}}$ that is entirely agnostic to the underlying base model class. This implies that even if the base model class $\mathcal{P}^{\textnormal{base}}$ exhibits arbitrarily large local extrapolation uncertainty, the proposed perturbed model maintains a guaranteed level of robustness. This upper bound is determined solely by the properties of the perturber $T$ and the distance $\delta$, where the latter quantifies the deviation of the test sequence from the training support $\mathcal{X}$. 
We give an explicit form of the constant in Theorem~\ref{thm: just_perturb} under a concrete scenario in Proposition~\ref{prop: extrap_synonym} (deferred to the end of this section).

Notably, Theorem~\ref{thm: just_perturb} holds without any conditions.
The main idea behind it is that the perturber $T$ ``smooths'' the next-token distribution, conditional on some prefix $\bm{X}$, by averaging over neighbouring prefixes $\widetilde{\bm{X}}$. Consequently, if two perturbed models $P_\gamma, P_{\gamma'} \in \mathcal{P}^{\textnormal{perturb}}$ agree on the training support in the sense that $\mathbb{TV}_{\bm{X}}(P_\gamma, P_{\gamma'}) =0$ for any $\bm{X} \in \mathcal{X}$, then their predictions cannot differ substantially at nearby unseen prefixes.

Our next theorem establishes a stronger result: an appropriately selected perturber not only offers robustness passively, but can actively promote the extrapolability of the class of perturbed models.

\begin{assumption}
\label{ass: SP_perturb}
For any $\delta>0$ and any $\bm{X}'$ with $d(\bm{X}', \mathcal{X}) \leq \delta$, the perturbation process is such that $T(\cdot| \bm{X}') = T(\cdot| \bm{X}_0)$ for some $\bm{X}_0\in \mathcal{X}$.
\end{assumption}

\begin{theorem}[Extrapolability]
\label{thm: extrap_class_charac_perturb}
Under Assumption~\ref{ass: SP_perturb}, we have $\mathcal{U}_{
    \text{\small 
    $\mathcal{P}^{\textnormal{perturb}}$
    }
    }(\delta)=0$ for any $\delta>0$, i.e., $\mathcal{P}^{\textnormal{perturb}}$ is extrapolable.
\end{theorem}

Theorem~\ref{thm: extrap_class_charac_perturb} strengthens the conclusion of Theorem \ref{thm: just_perturb} under a mild condition explained later. This theorem guarantees that the pretrained LLM under our framework can extrapolate perfectly for token sequences that are unseen in training, regardless of whether the base model class falls into regime (a), (b) or (c). 

The rationale behind this result is captured by Assumption~\ref{ass: SP_perturb}. Intuitively, a perturber satisfying this condition enables the model to map unfamiliar, out-of-support prefixes $\bm{X}'$ to familiar sequences $\bm{X}_0$ within the training support $\mathcal{X}$. This ensures that the model's predictions for $\bm{X}'$ are as reliable as those for $\bm{X}_0$. In essence, perturbation transforms an extrapolation task into an interpolation task by bridging the gap between the unseen $\bm{X}'$ and the seen $\bm{X}_0$.

Theorem~\ref{thm: extrap_class_charac_perturb} highlights the fundamental role of a well-designed perturber. In contrast to Theorem~\ref{thm: simple_extrap_feature_map}, which achieves perfect extrapolation through the extrapolability of the base model class, it provides a different route based entirely on the properties of the perturber.  
Together, Theorems~\ref{thm: simple_extrap_feature_map} and \ref{thm: extrap_class_charac_perturb} yield a double robustness property, which we summarize in the following corollary.

\begin{corollary}[Double robustness]\label{coro:dr}
    When either Assumption \ref{ass: UI} holds and $\mathcal{P}^{\textnormal{base}}$ is extrapolable, or Assumption~\ref{ass: SP_perturb} holds, $\mathcal{P}^{\textnormal{perturb}}$ becomes extrapolable. 
\end{corollary}
In summary, perturbation is adaptive when the base model already extrapolates perfectly, in the sense that it preserves perfect extrapolability and incurs no loss (Theorem \ref{thm: simple_extrap_feature_map}). At the same time, when the base model is not perfectly extrapolable, a suitably designed perturber can improve its efficiency (Theorem \ref{thm: contraction_global_extrap}), provide robustness when the base model extrapolates poorly (Theorem~\ref{thm: just_perturb}), and even recover perfect extrapolation (Theorem~\ref{thm: extrap_class_charac_perturb}). Taken together, these results reinforce the main takeaway of this work: \ul{\textit{perturbation is all you need for extrapolating your language model}}.

To conclude this section, we illustrate one practical example under which Assumption~\ref{ass: SP_perturb} is fulfilled. According to \citet{Satoetal2015}, it is reasonable to assume that token sequences are linked to multiple concrete source domains before subsequent language generation, and several sequences belong to the same concrete source domain. 
Formally put, let $\{D_1, \dots, D_d\}$ be a partition of $\mathcal{X}^*$ with $\mathcal{X}^* = \bigcup_{i=1}^d D_i$. For every $i$, let $Q_i$ be a probability distribution on perturbed prefixes, and suppose $T(\cdot|\bm{X}) = Q_i$ for every $\bm{X} \in D_i$.
Then Assumption~\ref{ass: SP_perturb} holds as long as the training support $\mathcal{X}$ consists of at least one sequence from every $D_i$, $i=1,\dots,d$.

\begin{proposition}[Extrapolability of token replacement]\label{prop: extrap_synonym}
Let $\mathcal{X}^*$ be equipped with Hamming distance,
and $Q_0$ be a specified distribution over the vocabulary $\mathcal{V}$\footnote{Note that $\mathcal{V}$ may include a ``null'' token and drawing this token represents deleting the original token.}.
Suppose the perturbation proceeds as follows.
\begin{itemize}
    \item [] Let $\ell \in [1,L]$ be given.
    Independently for the last $\ell$ tokens in input sequence $\bm{X}=(X_1, \dots, X_L)$: with probability $\beta\in (0,1)$, replace each $X_l$ by a random draw from $Q_0$; otherwise leave it intact.
\end{itemize}
Then with the notations in Theorems~\ref{thm: contraction_global_extrap} and \ref{thm: just_perturb}, we have 
$\alpha(T) \leq \min\{ (1-\beta)\ell, 1\}$ and
$C_{T,\delta} = \min\{ (1-\beta)\delta, 1\}$.
Moreover, if $Q_0$ has support outside of $\mathcal{X}$, then the corresponding class $\mathcal{P}^{\textnormal{perturb}}$ satisfies Assumption~\ref{ass: UI}.
\end{proposition}

\begin{remark}
Note that Assumption~\ref{ass: SP_perturb} can be readily relaxed.
In particular, suppose that for a given $\delta>0$ and any $\bm{X}'$ with $d(\bm{X}', \mathcal{X}) \leq \delta$, the perturbation process is such that $T(\cdot| \bm{X}') = T(\cdot| \bm{X}_0)$ for some $\bm{X}_0\in \mathcal{S} \subseteq \mathcal{X}$.
Then it holds that $\mathcal{U}_{\text{$\mathcal{P}^{\textnormal{perturb}}$}}^{\mathcal{S}}(\delta)=0$, where $\mathcal{U}_{\text{$\mathcal{P}^{\textnormal{perturb}}$}}^{\mathcal{S}}$ denotes the local extrapolation uncertainty from Definition~\ref{def: dist_extrap} with $\mathcal{X}$ replaced by $\mathcal{S}$.
With $\mathcal{S}$ arbitrarily small, the assumption on extrapolation is now more realistic, at the cost of how well we quantify the local extrapolation uncertainty.
\end{remark}

\begin{remark}\label{remark: extension}
For our setup in Definition~\ref{def: dist_extrap}, the total variation distance can be replaced by any distance between two measures, such as the Wasserstein distance, 
and similar results may be developed, provided the corresponding perturbation satisfies suitable regularity conditions.
In addition, Theorems~\ref{thm: simple_extrap_feature_map}--\ref{thm: extrap_class_charac_perturb}
can hold for general distributions (i.e., not necessarily for discrete random variable), defined on any measurable set that is not necessarily finite.
It is worth highlighting that all our theoretical guarantees developed are agnostic to both model architectures and estimation procedures,
independently of the specific parametrization or optimization algorithm used in practice.
Under additional regularity conditions, consistency and asymptotic theory for estimators in the perturbed model can be developed by formulating its training as an M-estimation problem with an integrated likelihood, see \cite{cen2026learning} for instance.
These arguments are largely standard and depend on the particular model parametrization, we thus choose not to pursue this line of analysis.
\end{remark}


\section{Simulation study}\label{sec:simulation}

In this section, we set the ground-truth language model $P^*$ to a bi-gram language model to generate synthetic datasets for evaluation. Specifically, we set $P^*$ to be such that $P^*(X_t|\bm{X}_{<t}) = P^*(X_t|X_{t-1})$. As a result, the synthetic data distribution is fully characterized by an initial token distribution $\pi \in \mathbb{R}^{|\mathcal{V}|}$ and a token-to-token transition matrix $M \in \mathbb{R}^{|\mathcal{V}| \times |\mathcal{V}|}$, where $|\mathcal{V}|$ denotes the vocabulary size. We set $\pi$ to a uniform distribution, and draw each row of the transition matrix $M$ independently from a Dirichlet distribution with a concentration parameter of $0.5$. We sample 500 sequences from $P^{*}$ to form a training dataset $\mathcal{D}$, where each sequence has a maximum length of 10.

To estimate $P^*$, we parametrize $P_\theta$ via a neural bi-gram language model proposed by \citet{zhang2023mixce} and learn the model parameters via MLE. In detail, each token is first mapped to a $d$-dimensional continuous vector as one row in an embedding matrix $\mathbf{E} \in \mathbb{R}^{|\mathcal{V}| \times d}$. Given the embedding vector $E_{i-1}$ of the previous token $X_{i-1}$, the probability of next token is predicted via a two-layer neural network:
\begin{align*}
    P_{\theta}(X_i | X_{i-1}) =
    \overbrace{
        \textup{SoftMax}\Big(
        \mathbf{W}_2
        \underbrace{
            \textup{Dropout}\!\left(
            \textup{ReLU}(\mathbf{W}_1 E_{i-1} + \mathbf{b}_1)
            \right)
        }_{\textup{the first layer}}
        + \mathbf{b}_2
        \Big)
    }^{\textup{the second layer}},
\end{align*}
where the parameters $\theta$ include $\mathbf{W}_1$, $\mathbf{W}_2$, $\mathbf{b}_1$, and $\mathbf{b}_2$. We fix the embedding dimension as $d=50$ and the dropout rate as $0.1$ when estimating $\theta$. After obtaining the estimated $\theta$, we feed the embedding matrix $\mathbf{E}$ into $P_{\theta}$ to directly obtain the estimated transition matrix. 

Our procedure relies on a perturbation intensity parameter that controls the variability of the perturbation. Larger values of this parameter result in more tokens in the input prefix being replaced to generate perturbed samples. In the extreme case where the parameter is set to zero, the procedure reduces to standard pretraining without perturbation. 
More details on $P_{\theta}$ and our training procedures are deferred to Section \ref{subsec:config}. Since $P_\theta$ yields an estimated transition matrix $\widehat{M}$ and the oracle matrix $M$ is known, we evaluate the model's extrapolation capability by computing mean absolute error (MAE) between $\widehat{M}$ and $M$ over token pairs whose first elements are absent from the training set $\mathcal{D}$.

\begin{figure}[t]
    \centering
    \includegraphics[width=0.9\linewidth]{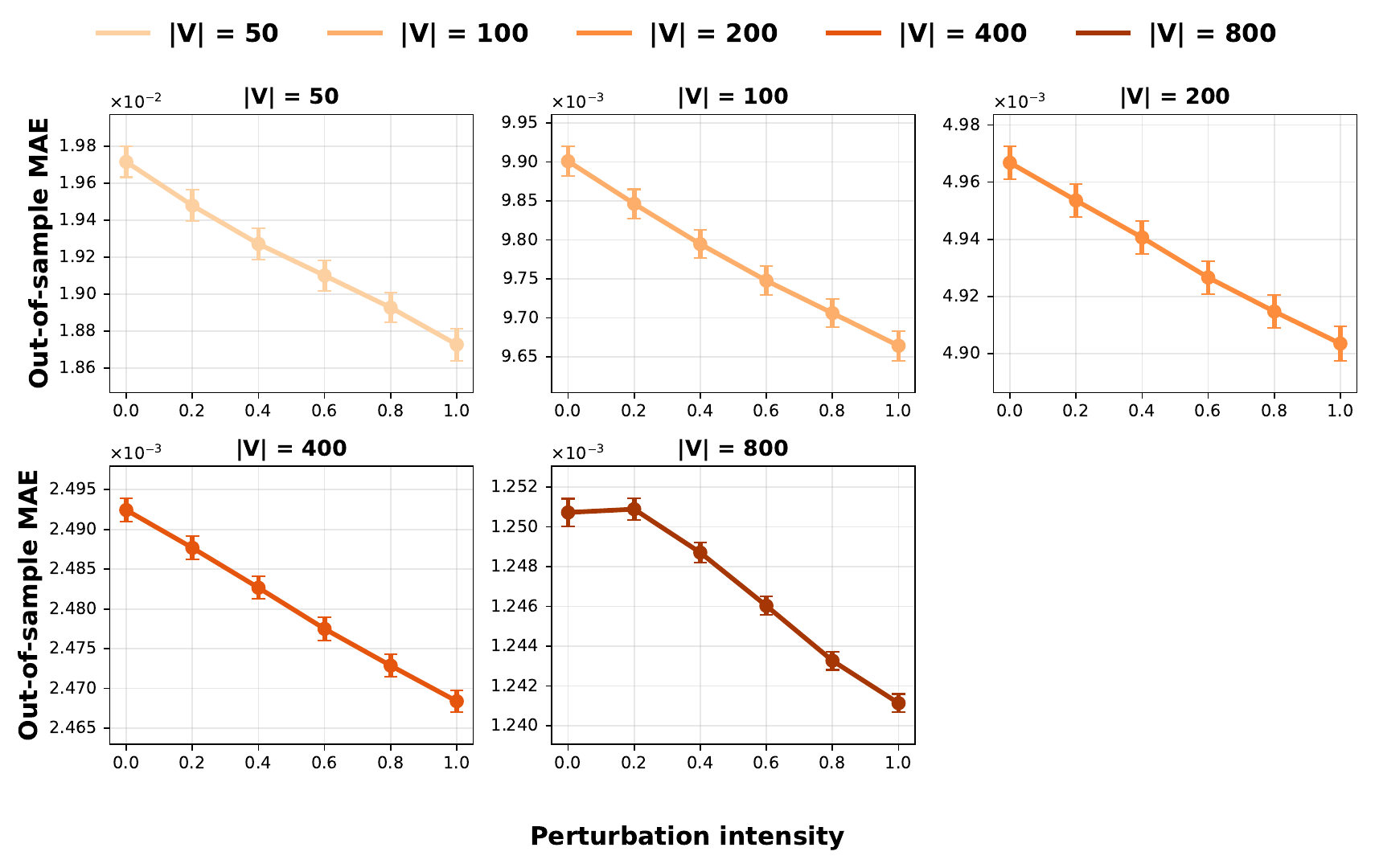}
    \caption{MAE of estimated transition matrices on unobserved token pairs with varying perturbation intensities. From top-left to bottom-right, the vocabulary size $|\mathcal{V}|$ increases from 50 to 800; a smaller $|\mathcal{V}|$ represents a dense data regime where training samples provide sufficient coverage. Results reported are mean MAEs $\pm$ 2 standard errors, aggregated over 100 independent replications.
    }
    \label{fig:sim}
\end{figure}

Figure~\ref{fig:sim} visualizes the MAEs of our learned models when varying the perturbation intensity parameter and the number of tokens $|\mathcal{V}|$ in the vocabulary. 
The top-left panel of Figure~\ref{fig:sim} represents a regime with a small vocabulary size, where the training set $\mathcal{D}$ sufficiently covers the text space. 
As $|\mathcal{V}|$ increases, our approach significantly reduces the MAE compared to the classical approach, and the improvement becomes increasingly pronounced as the perturbation intensity increases. These findings empirically verify that semantic perturbation can significantly bolster the extrapolation capabilities of LLMs.

\section{Real data analyses}\label{sec:real-world}

\subsection{A real-world application on language modeling}

In this subsection, we follow the experiment setup in \citet{ren2024emo} and use the WikiText-2 dataset \citep{merity2017pointer} for training. 
Details about the setup and data are reported in Section~\ref{subsec:config}.
We evaluate the learned model's performance on three \textit{out-of-sample} datasets: (i) WikiText-103, which is more than 50 times larger than WikiText-2; (ii) the test split of the WebText dataset released by OpenAI \citep{radford2019language}; and (iii) WritingPrompts, drawn from the writing prompts forum on Reddit \citep{fan2018hierarchical}. We also use the WikiText-2 test set to assess the \textit{in-sample} performance of the learned model.

We implement our method as described in Section~\ref{sec: method}, using random insertion to generate perturbations; see more variants and their results in the Supplementary Materials. 
We compare our method against five baseline algorithms: standard training by MLE, the Brier score maximizing, MixCE \citep{zhang2023mixce}, total variation distance minimization \citep[TVD,][]{ji2023tailoring}, and Earth Mover’s Distance optimization \citep[EMO,][]{ren2024emo}. For a fair comparison, all baseline methods use the same model architecture described below. Our approaches are built on the same base model architecture, with the addition of perturbation.

We assess the performance of the various training objectives across three representative decoder-only Transformer architectures: OPT-125M \citep{zhang2022opt}, GPT-2 \citep{radford2019language}, and GPT-Neo-1.3B \citep{gptneo}. 
The implementation details are deferred to Section \ref{subsec:config}. 
We use Mauve \citep{pillutla2023mauve} as our primary evaluation criterion, which measures the alignment between model-generated completions and human-written text by computing the area under the KL-divergence frontier and is widely adopted in the literature. In addition, we report the ROUGE-1 score \citep{lin2004rouge}, which rewards models for maintaining high $1$-gram overlap.

\begin{table*}[t]
\centering
\normalfont
\setlength{\tabcolsep}{3.5pt}
\caption{Results for various combinations of models and datasets, ROUGE-1 is written as R-1 for short. The best scores for each metric are highlighted in \text{bold}.
Results of our methods are denoted by ``MLE (ours)''. The largest standard errors for Mauve and R-1 are 0.0481 and 0.001, respectively.
}
\label{tab:language_model_multi}
\begin{tabular}{llcccccccc}
\toprule
\multirow{2}{*}{Model} & \multirow{2}{*}{Method} & \multicolumn{2}{c}{WikiText-2} & \multicolumn{2}{c}{WikiText-103} & \multicolumn{2}{c}{WebText} & \multicolumn{2}{c}{WritingPrompts} \\
 &  & Mauve & R-1 & Mauve & R-1 & Mauve & R-1  & Mauve & R-1  \\
\midrule
\multirow{7}{*}{OPT-125M} & MLE & 0.83 & 0.38 &  0.80 & 0.38  & 0.44 & 0.37 &  0.11 & 0.36  \\
 & MixCE & 0.77 & 0.38 &  0.79 & 0.38  & 0.46 & 0.37  & 0.10 & 0.36  \\
 & TVD & 0.78 & 0.38  & 0.77 & 0.38 & 0.40 & 0.37 & 0.10 & 0.36  \\
 & Brier & 0.71 & 0.37  & 0.67 & 0.37  & 0.40 & 0.36  & 0.11 & 0.36  \\
 & EMO & 0.80 & 0.38  & 0.79 & 0.38  & 0.49 & 0.37  & 0.10 & 0.36  \\
 & MLE (ours) & \textbf{0.90} & \textbf{0.41}  & \textbf{0.87} & \textbf{0.41}  & \textbf{0.56} & \textbf{0.39}  & \textbf{0.13} & \textbf{0.38}  \\
\midrule
\multirow{7}{*}{GPT-2} & MLE & 0.74 & 0.37  & 0.74 & 0.37  & 0.47 & 0.37  & 0.10 & 0.36  \\
 & MixCE & 0.77 & 0.38  & 0.77 & 0.38 & 0.48 & 0.37  & 0.10 & 0.36  \\
 & TVD & 0.75 & 0.37  & 0.75 & 0.37  & 0.45 & 0.37  & 0.09 & 0.36 \\
 & Brier & 0.71 & 0.36  & 0.74 & 0.36 & 0.49 & 0.36  & 0.11 & 0.36  \\
 & EMO & 0.82 & 0.39  & 0.82 & 0.39  & \textbf{0.58} & 0.38  & 0.13 & 0.37  \\
 & MLE (ours) & \textbf{0.89} & \textbf{0.40}  & \textbf{0.87} & \textbf{0.40}  & 0.57 & \textbf{0.39}  & \textbf{0.17} & \textbf{0.38} \\
\midrule
\multirow{7}{*}{GPT-Neo-1.3B} & MLE & 0.82 & 0.39  & 0.82 & 0.39  & 0.54 & 0.38  & 0.14 & 0.36  \\
 & MixCE & 0.79 & 0.39  & 0.79 & 0.39 & 0.53 & 0.38  & 0.14 & 0.37  \\
 & TVD & 0.81 & 0.39  & 0.81 & 0.39  & 0.53 & 0.38 & 0.14 & 0.37  \\
 & Brier & 0.79 & 0.38  & 0.79 & 0.38  & 0.51 & 0.37  & 0.13 & 0.36  \\
 & EMO & 0.85 & 0.38  & 0.85 & 0.38  & 0.56 & 0.36  & 0.13 & 0.35  \\
 & MLE (ours) & \textbf{0.87} & \textbf{0.41}  & \textbf{0.87} & \textbf{0.41}  & \textbf{0.58} & \textbf{0.39}  & \textbf{0.17} & \textbf{0.38} \\
\bottomrule
\end{tabular}
\end{table*}

The results are summarized in Table~\ref{tab:language_model_multi}. 
Our proposed method attains the highest ROUGE-1 scores in all reported combinations, and the highest Mauve score in all but GPT-2 on WebText, where our Mauve score is comparable to, but slightly below, that of EMO.
The results on the WikiText-2 test set show that our approach is also beneficial in interpolation, while those 
on WebText and WritingPrompts demonstrate its superior out-of-distribution extrapolation capability. 
These findings provide strong empirical evidence for the effectiveness of incorporating perturbations into training. 
We note that our improvements in ROUGE-1 scores are relatively minor, which is not too surprising given that the ROUGE-1 score measures 1-gram overlap and thus assesses how well the prediction matches lexical rather than semantic meanings. Thus, it favours verbose and repetitive predictions, which might be overly simple.

Lastly,
we showcase some examples of the generated texts. Table~\ref{tab:sample_human_140} presents the results of different methods by training GPT-2 on WikiText-2, with WebText being the test set. It can be seen that baseline methods generate some unexpected punctuations in the texts. In contrast, our methods generate more coherent texts.

\definecolor{lightgray}{RGB}{245, 245, 245}
\definecolor{wrongred}{RGB}{231, 76, 60}
\definecolor{headerblue}{HTML}{1F5AA6}
\definecolor{correctgreen}{HTML}{1B9E77} 

\begin{table*}[htbp]
\centering
\scriptsize
\renewcommand{\arraystretch}{1.2}
\caption{Examples of texts generated by different models when tested on the WebText dataset.}
\label{tab:sample_human_140}
\begin{tabular}{>{\raggedright\arraybackslash}p{2.5cm} p{11cm}}
\toprule
\rowcolor{headerblue!30}
\multicolumn{2}{c}{\textbf{Human Source}} \\
\midrule
\textbf{Input text} & Clinton talks about her time of ``reflection'' during sick days Hillary Clinton returned to the campaign trail Thursday afternoon, \\


\midrule
\rowcolor{headerblue!10}
\multicolumn{2}{c}{\textbf{Predictions}} \\
\midrule

\rowcolor{lightgray} \textbf{Ground truth} & debuting some new intro music and telling the crowd that her sick days allowed her a chance to ``reconnect with what this whole campaign is about." The former secretary of state, who took the stage to James Brown's ``I Feel Good," spent the beginning of the week at her home in Chappaqua, New York, after being diagnosed late last week with pneumonia.  \\

\midrule

\textbf{MLE}   & \textit{in Arkansas. Noting Sanders' performance in the Senate, she took a wide laugh. Clinton's presidential campaign ``lifted the veil of hypocrisy'' in the long days following the Iowa caucuses. In an earnings call to Goldman Sachs executives in early November, Clinton remarked that networking with Republicans ``would be tough, but James [ Risen ] is an important role model in ... .''} \\

\textbf{MixCE} & \textit{and the senator polices the stage and promises everyone that the thoughts of her former colleagues – not to mention more Youngstown fans – will continue to serve her well. The night before, the Democratic National Convention extended south @-@ west to Boston, with an extension from Boston to the Long Branch City Center in the downtown area. = = = Reactions = = =} \\

\textbf{TVD} & \textit{and the senator polices the stage and promises everyone that the thoughts of her former colleagues – not when she was in touch with her husband – say standing up for the middle class : " It is about time for the United States to take a look back in time — and resolve — the financial crisis of 2008. " = = Political history = = } \\

\textbf{Brier} & \textit{in Arkansas. Noting Sanders' performance in the California primary, she recovered from the blow her states sent her. Noting " four years ago, Hillary Clinton led in national surveys. She kept qualified. This is a legacy to the voters who voted her against " — Mayor Clinton (@-@ Rahm Emanuel ), @-@ DNC I. M. [ Hillary Clinton @-@ Wasserman Schultz ]} \\

\textbf{EMO} & \textit{debuting her campaign website in early June. The campaign unveiled the campaign's five @-@ minute video, which takes an intimate look at what the campaign has been up to since its introduction, as well as highlights of Clinton's political future.} \\

\midrule

\textbf{MLE (ours)} & \textit{debuting a new ``The Bill Clinton Experience'' in which she describes how they made up their minds to help each other become more successful before a campaign. The audio in the video also features Clinton's former Secretary of State, Hillary Rodham Clinton, who talked with him during the Clinton family's family vacations.} \\


\bottomrule
\end{tabular}
\end{table*}

\subsection{An ablation study}\label{sec:ablation}

As discussed earlier, our method employs perturbations during both training and inference. To evaluate the contribution of perturbation at each stage, we conduct an ablation study using the real data described in this section. We compare our approach against three variants: 
\begin{enumerate}
    \item [(i)] \textit{NoPerturb}, the standard approach without perturbations during either training or inference;
    \item [(ii)] \textit{TrainPerturb}, which trains the model with perturbed data but performs inference without perturbations, and is closely related to data augmentation;
    \item [(iii)] \textit{TestPerturb}, which trains the model on the original data while applying perturbations only during inference.
\end{enumerate}

\begin{figure}[t]
    \centering
    \includegraphics[width=0.64\linewidth]{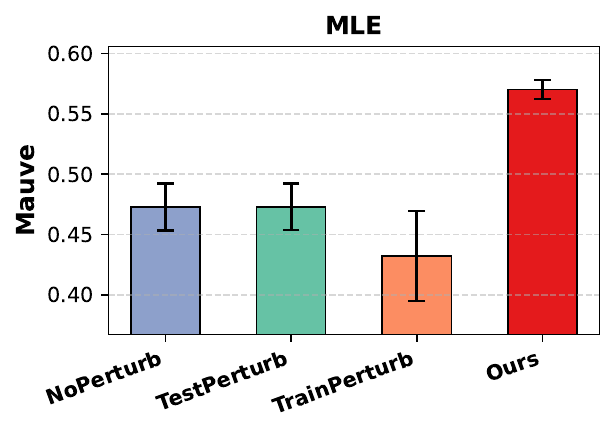}
    \caption{The ablation study on real-world data analysis, showing the mean Mauve scores $\pm$ 2 standard errors across 100 repetitions. 
    The study is conducted
    with GPT-2 as the base model, WikiText-2 as the training dataset, and the trained model evaluated on the WebText dataset with the Mauve metric.
    }\label{fig:ablation-mle}
\end{figure}

Figure~\ref{fig:ablation-mle} visualizes the Mauve scores of the three variants above and our method, using the objective function in \eqref{eqn: obj_log}. It can be seen that neither TrainPerturb nor TestPerturb significantly outperforms the standard approach (i.e., without perturbation), indicating that applying perturbation at only one stage (training or inference) is insufficient to achieve meaningful performance gains. In contrast, our approach, which integrates perturbation during both training and inference, achieves substantial improvements, demonstrating the importance of perturbing at both stages.

In addition, it is also interesting to see that the performance of TrainPerturb is generally unsatisfactory (compared to NoPerturb), despite of more samples used in the training stage.
However, this might not be surprising if we deem perturbation as a part of LLM architecture: TrainPerturb omits this component at the inference stage.

\subsection{Experimental details}
\label{subsec:config}
We detail the data, baseline methods, training and evaluation in this section.

\textbf{Source of datasets.} Two datasets, WikiText-2 and WikiText-103, are available from the Hugging Face dataset repository \texttt{Salesforce/wikitext}\footnote{\url{https://huggingface.co/datasets/Salesforce/wikitext}}. As for WritingPrompts and WebText, they can be accessed via the interface provided in the Github repository \texttt{bloomberg/MixCE-acl2023}\footnote{\url{https://github.com/bloomberg/MixCE-acl2023}}.

\textbf{Methods}. MixCE involves a tuning parameter that convexly combines the forward KL and reverse KL divergences; we fix this parameter to 0.5 in all experiments. TVD also includes a parameter that trades off bias and variance, which we also set to 0.5. MLE, EMO, and Brier do not involve any tuning parameters. As mentioned in Section~\ref{sec:simulation}, our procedure requires to specify a perturbation intensity parameter, which we fix to 0.025 throughout the experiments.

\textbf{Training details}. For fairness, all methods apply the same training configuration described below. We use the AdamW optimizer \citep{loshchilov2018decoupled} with a learning rate of $5\times 10^{-5}$. Besides, we adopt a linear learning rate scheduler throughout training, which is trained with gradient accumulation steps set to 1. The batch size is fixed to 32. The maximum input length during training is set to 256. 

\textbf{Evaluation details}. To gauge the quality of the learned model distribution $P_\theta$ in a faithful way \citep{eikema2020map}, we employ unbiased sampling as the primary decoding algorithm throughout the experiments as in  \citet{ren2024emo}. The lengths of prefix sequences for each dataset are summarized in Table~\ref{tab:length-prefix}, and the length of the generated tokens is fixed as 80. We repeat the sampling process for 20 independent runs for each prefix and report the average Mauve score.

\begin{table}[htbp]
\centering
\caption{Length of the provided prefix for each dataset used in the real-world language modeling  experiments.}\label{tab:length-prefix}
\begin{tabular}{lcccc}
\toprule
 & \text{WikiText2} & \text{Wikitext103} & \text{WebText} & \text{WritingPrompts} \\
Length      & 20 & 20 & 20 & 35   \\
\bottomrule
\end{tabular}
\end{table}

\section{Concluding remarks}\label{sec: conclusion}
We introduce a novel perturbed model class for LLM pretraining. Recent work has shown growing interest in developing statistical methodologies and foundations for LLMs, including test-time scaling \citep{li2026should}, post-training from human feedback \citep{lee2024lowrankcontextualreinforcementlearning, liu2024dual, liu2025statistical, lu2025contextual, xiao2025algorithmic, ye2025robust, cho2026privacy, fang2026variance} or AI feedback \citep{xia2026statistical}, reasoning \citep{gong2026kernelized, huang2026learning, zhang2026reasoning, zhou2026demystifying}, ranking and evaluation \citep{fan2025uncertainty, fan2025ranking, liu2025uncertaintyquantificationlargelanguage, Lietal2026}, LLM-generated text detection \citep{li2025statistical, xie2025debiasing, li2026robust, sun2026curverl, zhou2026detecting} and LLM-assisted statistical modeling \citep{gao2026does}. Nonetheless,  pretraining remains underexplored. Our work contributes to the literature by addressing this gap.

The proposed perturbed model can be conveniently trained and deployed. Theoretically, we demonstrate that the class of perturbed models is adaptive, robust, and can extrapolate. Numerically, our framework is also corroborated by the consistent empirical improvements of the pretrained LLMs in the reported settings. Although we choose to directly specify the perturber in this paper, it is possible to consider a parametrized perturbation class. These parameters could be jointly optimized alongside the LLM's primary parameters to further enhance learning. This extension has been recently investigated by   \citet{cen2026learning}. Also, we expect such perturbed model class can be extended from LLMs to visual- and multimodal-LLMs. We leave the detailed development to future research.


\section*{Disclosure statement}\label{disclosure-statement}
The authors report there are no competing interests to declare.

\section*{Data availability statement}\label{data-availability-statement}
Data are publicly available, with relevant links included in the paper.


\bibliographystyle{apalike}
\bibliography{reference.bib}	


\clearpage


\setcounter{section}{0}
\setcounter{equation}{0}
\def\theequation{S\arabic{equation}}
\def\thesection{S\arabic{section}}
\def\thetheorem{S\arabic{theorem}}
\def\thelemma{S\arabic{lemma}}
\def\thefigure{S\arabic{figure}}
\def\thedefinition{S\arabic{definition}}
\def\theexample{S\arabic{example}}
\def\theremark{S\arabic{remark}}
\def\thetable{S\arabic{table}}
\def\thecorollary{S\arabic{corollary}}
\def\theassumption{S\arabic{assumption}}
\def\theproposition{S\arabic{proposition}}

\urlstyle{same}
\numberwithin{equation}{section}

\begin{center}
{\Large\textbf{
Supplement to ``Perturbation is All You Need for Extrapolating Language Models''
}}
\end{center}

\medskip

In this supplementary material, we provide additional details in the main text. 
We include proofs of the main results stated in the main text, together with auxiliary results and their proofs.
Additional discussions and numerical results are also included.

\section{Additional numerical results}\label{appendix:additional-numeric}

\subsection{Scoring rules as objective functions}\label{subsec: scoring_rule_objective}

The parameter $\theta$ in LLMs can be computed by optimizing an objective function, for example, a proper scoring rule \citep{GneitingRaftery2007}.
A \textit{scoring rule} $S(P,v)$ measures how well a token-level distribution $P$ aligns with a specific token $v$. The corresponding \textit{expected score} is defined as $S(P,P^*) := \mathbb{E}_{v \sim P^*}[S(P,v)]$.
A scoring rule is \textit{proper} if $S(P^*,P^*) \geq S(P,P^*)$ for any distributions $P,P^*$. It is \textit{strictly proper} if the equality holds if and only if $P=P^*$. Given a strictly proper scoring rule, we are guaranteed to recover $P^*$ by maximizing the expected score $S(P,P^*)$ (or minimizing the loss $-S(P,P^*)$), since the optimizer exists and is unique. 

A prominent example is the logarithmic scoring rule, defined as $S(P,v)= \log P(v)$. Maximizing this score over a data sequence $\bm{X}$ optimizes the log-likelihood of the data $\sum_t S(P_{\theta}(\cdot|\bm{X}_{<t}), X_t)=\sum_t \log P_{\theta}(X_t|\bm{X}_{<t})$, which is equivalent to minimizing the cross-entropy loss.
The main text shows exactly the case when the logarithmic scoring rule is used.

Another popular family of strictly proper scoring rules is the $\alpha$-power score \citep{Selten1998}. It is defined as
$S(P,v)= \alpha P^{\alpha-1}(v) - (\alpha-1) \| \bm{P} \|_\alpha^{\alpha}$
for a given constant $\alpha>1$, where $\bm{P}$ denotes the vector $[P(v_1),P(v_2),\cdots]^\top$ that contains the probabilities for all tokens in the vocabulary $\mathcal{V}$, and $\|\cdot\|_\alpha$ denotes the $L_\alpha$-norm of the vector.
In the special case where $\alpha=2$, it is reduced to the Brier score \citep{Brier1950}, and maximizing the cumulative score $\sum_t S(P_{\theta}(\cdot|\bm{X}_{<t}), X_t)$ over a data sequence $\bm{X}$ is mathematically equivalent to minimizing the squared Euclidean distance between $P_{\theta}$ and the ground-truth distribution $P^*$, $\sum_t \|\bm{P}_{\theta}(\cdot|\bm{X}_{<t})-\bm{P}^*(\cdot|\bm{X}_{<t})\|_2^2$.

For any scoring rule, the model can be trained as follows.
Given a prefix $\bm{X}_{<t}$ from a training sequence $\bm{X}$, we first sample $m$ number of perturbed sequences $\{\widetilde{\bm{X}}_{<t}^{(i)}\}_{i=1}^m$ such that 
\[
\widetilde{\bm{X}}_{<t}^{(i)} \sim T(\cdot| \bm{X}_{<t}) ,
\quad
i=1,\dots,m .
\]
For each perturbed sequence, we compute their individual score $S(P_{\theta}(\cdot|\widetilde{\bm{X}}_{<t}^{(i)}), X_t)$ by a strictly proper scoring rule $S$. These $m$ scores are averaged to approximate the expectation 
\begin{align*}
    \mathbb{E}_{\widetilde{\bm{X}}_{<t} \sim T(\cdot|\bm{X}_{<t})}S(P_{\theta}(\cdot|\widetilde{\bm{X}}_{<t}), X_t)
\end{align*}
under the perturbation distribution $T$. By aggregating these averaged scores over $t$, it yields our objective function for a single sequence, 
\begin{eqnarray}\label{eqn: loss_score_population}
    \mathcal{L}^{S}(P_\theta; \bm{X})=\frac{1}{m}\sum_{t=1}^L \sum_{i=1}^m S[P_\theta(\cdot| \widetilde{\bm{X}}_{<t}^{(i)}), X_t].
\end{eqnarray}
Then we aggregate $\mathcal{L}^{S}(P_\theta; \bm{X})$ over all sequences $\bm{X}$ in a training corpus $\mathcal{D}$, and parameter $\theta$ is estimated by maximizing this objective function $\sum_{\bm{X}\in \mathcal{D}}\mathcal{L}^{S}(P_\theta; \bm{X})$.  
For a concrete example when the Brier scoring rule is employed, \eqref{eqn: loss_score_population} respectively reduces to:
\begin{equation}
\label{eqn: obj_brier}
\begin{split}
    &\mathcal{L}^{\text{Brier}}(P_\theta; \bm{X}) :=
    \frac{1}{m}
    \sum_{i=1}^m\sum_{t=1}^L \left[ 2 P_\theta(X_t | \widetilde{\bm{X}}_{<t}^{(i)}) - \|\bm{P}_\theta(\cdot | \widetilde{\bm{X}}_{<t}^{(i)})\|_2^2 \right].
\end{split}
\end{equation}

\subsection{Additional numerical results}

\textbf{Additional real data analysis: improvement by perturbation}.
As EMO in Table~\ref{tab:language_model_multi} generally prevails over other unperturbed frameworks, we further experiment the EMO method under perturbation. 
The results are presented in Table~\ref{tab:language_model_multi_before_after}, together with those from the vanilla EMO method for convenience.
We also experiment on the setting employing the Brier score (as detailed in Appendix~\ref{subsec: scoring_rule_objective}); see Table~\ref{tab:language_model_multi_before_after_brier}.
The implication for those results are similar to that of Table~\ref{tab:language_model_multi}, i.e., the methods can be visually improved under such simple perturbation.

\begin{table*}[ht]
\centering
\normalfont
\setlength{\tabcolsep}{3.5pt}
\caption{Results of vanilla EMO and perturbed EMO (denoted by ``(ours)'') for various combinations of models and datasets, ROUGE-1 is written as R-1 for short. The best score for each metric is highlighted in \text{bold}. The largest standard errors for Mauve and R-1 are 0.0266 and 0.001, respectively.
}
\label{tab:language_model_multi_before_after}

\begin{tabular}{llcccccccc}
\toprule
\multirow{2}{*}{Model} & \multirow{2}{*}{Method} & \multicolumn{2}{c}{WikiText-2} & \multicolumn{2}{c}{WikiText-103} & \multicolumn{2}{c}{WebText} & \multicolumn{2}{c}{WritingPrompts} \\
 &  & Mauve & R-1  & Mauve & R-1  & Mauve & R-1  & Mauve & R-1 \\
\midrule
\multirow{2}{*}{OPT-125M} & EMO & 0.80 & 0.38  & 0.79 & 0.38  & 0.49 & 0.37 & 0.10 & 0.36  \\
 & EMO (ours) & \textbf{0.91} & \textbf{0.40} & \textbf{0.90} & \textbf{0.41} & \textbf{0.61} & \textbf{0.39} & \textbf{0.14} & \textbf{0.38} \\
\midrule
\multirow{2}{*}{GPT-2} & EMO & 0.82 & 0.39  & 0.82 & 0.39  & \textbf{0.58} & 0.38 & 0.13 & 0.37 \\
 & EMO (ours) & \textbf{0.89} & \textbf{0.41} & \textbf{0.90} & \textbf{0.40} & \textbf{0.58} & \textbf{0.40} & \textbf{0.17} & \textbf{0.38} \\
\midrule
\multirow{2}{*}{GPT-Neo-1.3B} & EMO & 0.85 & 0.38 & 0.85 & 0.38  & 0.56 & 0.36 & 0.13 & 0.35  \\
 & EMO (ours) & \textbf{0.88} & \textbf{0.40} & \textbf{0.89} & \textbf{0.40} & \textbf{0.62} & \textbf{0.38} & \textbf{0.18} & \textbf{0.36} \\
\bottomrule
\end{tabular}
\end{table*}

\begin{table*}[ht]
\centering
\normalfont
\setlength{\tabcolsep}{3.5pt}
\caption{Results of vanilla Brier and perturbed Brier (denoted by ``(ours)'') for various combinations of models and datasets, ROUGE-1 is written as R-1 for short. The best score for each metric is highlighted in \text{bold}. The largest standard errors for Mauve and R-1 are 0.0405 and 0.002, respectively.
}
\label{tab:language_model_multi_before_after_brier}

\begin{tabular}{llcccccccc}
\toprule
\multirow{2}{*}{Model} & \multirow{2}{*}{Method} & \multicolumn{2}{c}{WikiText-2} & \multicolumn{2}{c}{WikiText-103} & \multicolumn{2}{c}{WebText} & \multicolumn{2}{c}{WritingPrompts} \\
 &  & Mauve & R-1  & Mauve & R-1  & Mauve & R-1  & Mauve & R-1 \\
\midrule
\multirow{2}{*}{OPT-125M} & Brier & 0.71 & 0.37  & 0.67 & 0.37  & 0.40 & 0.36  & 0.11 & 0.36  \\
 & Brier (ours) & \textbf{0.89} & \textbf{0.40}  & \textbf{0.91} & \textbf{0.40}  & \textbf{0.60} & \textbf{0.39} & \textbf{0.15} & \textbf{0.38}  \\
\midrule
\multirow{2}{*}{GPT-2} & Brier & 0.71 & 0.36  & 0.74 & 0.36 & 0.49 & 0.36  & 0.11 & 0.36  \\
 & Brier (ours) & \textbf{0.90} & \textbf{0.40}  & \textbf{0.91} & \textbf{0.40}  & \textbf{0.67} & \textbf{0.39}  & \textbf{0.19} & \textbf{0.38}  \\
\midrule
\multirow{2}{*}{GPT-Neo-1.3B}  & Brier & 0.79 & 0.38  & 0.79 & 0.38  & 0.51 & 0.37  & 0.13 & 0.36  \\
 & Brier (ours) & \textbf{0.91} & \textbf{0.41}  & \textbf{0.88} & \textbf{0.41} & \textbf{0.65} & \textbf{0.39} & \textbf{0.22} & \textbf{0.38} \\
\bottomrule
\end{tabular}
\end{table*}

\begin{table}[htbp]
\centering
\caption{Mauve comparison of three perturbation strategies when perturbation intensity varies. The comparison is based on the following experimental setting: GPT-2 is trained on WikiText-2 using the proposed method with three different perturbation strategies and evaluated on the WebText test set.}\label{tab:diff-perturb}
\begin{tabular}{cccc}
\toprule
Perturbation intensity & Replacement & Deletion & Insertion \\
\midrule
0.0000   & 0.505 & 0.498 & 0.507 \\
0.0125  & 0.648 & 0.666 & 0.657 \\
0.0250  & 0.664 & 0.671 & 0.695 \\
0.0500  & 0.653 & 0.668 & 0.678 \\
\bottomrule
\end{tabular}
\end{table}

\textbf{Additional ablation study}.
On top of the ablation study presented in the main text, we conduct an additional experiment using the Brier score (see Appendix~\ref{subsec: scoring_rule_objective}).
The result is showcased in Figure~\ref{fig:ablation-brier}, once again supporting the effect of perturbation in both training and inference stages.

\begin{figure}[!t]
    \centering
    \includegraphics[width=0.6\linewidth]{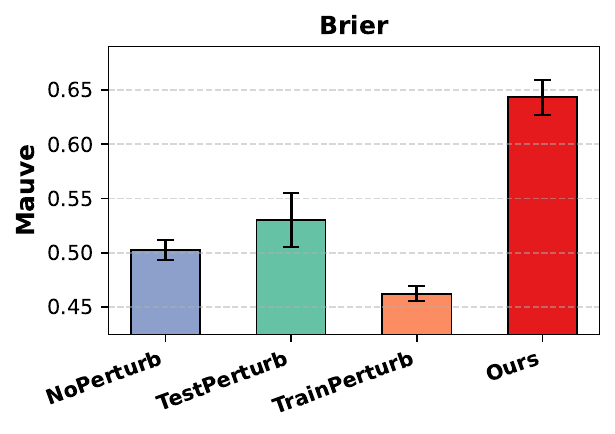}
    \caption{The ablation study on real-world data analysis, showing the mean Mauve scores $\pm$ 2 standard errors across 100 repetitions.
    The study is conducted
    with GPT-2 as the base model, WikiText-2 as the training dataset, and the model trained using the Brier score (see Section~\ref{subsec: scoring_rule_objective}), and evaluated on the WebText dataset with the Mauve metric.
    }\label{fig:ablation-brier}
\end{figure}

\textbf{Effects of perturbation design}.
In addition to the perturbation method by random insertion which we adopted in the main text, we also experiment synonym replacement and random deletion. We present the results in Table~\ref{tab:diff-perturb}, from which we observe that, overall, the three perturbation methods do not exhibit substantial differences in performance.

\textbf{Sensitivity to perturbation intensity}.
We now study the effect of perturbation intensity in real-world language modeling. The training dataset is WikiText-2, and the evaluation dataset is WebText. We apply the proposed framework to train a GPT-2 model using the random insertion. The perturbation intensity $\alpha$ is set to $0.1 \times 2^{\kappa}$, where $\kappa \in \{-4, -3, -2, -1, 0, 1, 2\}$. 
When $\alpha$ is close to zero, the proposed method approaches the classical method, providing no gain in extrapolability. In contrast, when $\alpha$ is close to one, overly strong perturbations may produce training data that substantially deviates from the underlying human language distribution; as a result, leveraging such heavily perturbed data can degrade performance.

\begin{figure}[h!]
    \centering
    \includegraphics[width=0.6\linewidth]{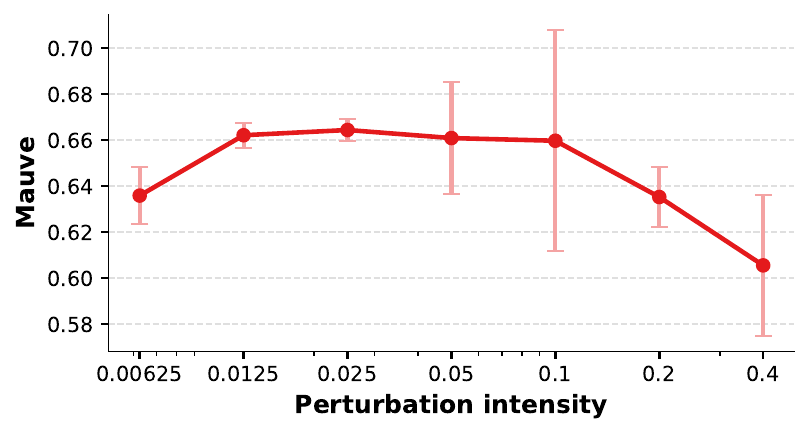}
    \caption{The perturbation intensity versus Mauve plot, shown as mean Mauve $\pm$ 2 standard errors as in Figure~\ref{fig:sim}.}
    \label{fig:perturb-strength}
\end{figure}

The results reported in Figure~\ref{fig:perturb-strength} support the above discussion.
By choosing a moderate perturbation intensity, the proposed method achieves improved explorability in the trained model. Based on Figure~\ref{fig:perturb-strength}, we find that $\alpha = 0.025$ is an effective choice. We adopt this setting in real-world analysis in Section~\ref{sec:real-world}, where it generally yields performance improvements over baselines across various combinations of LLMs and datasets (see Table~\ref{tab:language_model_multi}).

\section{Meta-theoretical results and all the proofs}
\label{appendix: formal_theory_w_discuss}

\subsection{Additional results and discussion}

We include in this section the general statements of all definitions and theoretical results presented in the main text, with additional discussion.
All the proofs are included in Appendix~\ref{appendix: formal_theory_w_discuss}.
To help readers navigate among the statements and results in this paper, we also present Table~\ref{tab: statement_summary} as a summary.

\begin{table}[!t]
    \centering
    \caption{Summary of theoretical statements and results.}
    \label{tab: statement_summary}
    \begin{tabular}{ll}
       Name  &  Description/Interpretation  
       \\
       \midrule
       Definition~\ref{def: dist_extrap} & Definitions of extrapolation uncertainty and extrapolability 
       \\
       Assumption~\ref{ass: UI} & An identifiability condition which guarantees zero perturbation cost 
       \\
       Assumption~\ref{ass: SP_perturb} & A sufficient condition to extrapolability on the perturbed model class  
       \\
       Assumption~\ref{ass: SP} & A sufficient condition to extrapolability on a general model class 
       \\
       Assumption~\ref{ass: zero_SCE} & Imposing zero perturbation cost 
       \\
       Lemma~\ref{lemma: perturbation_inference} & A useful representation to aid studying extrapolability 
       \\
       Lemma~\ref{lemma: UI_imply_zero_cost} & Assumption~\ref{ass: UI} implies Assumption~\ref{ass: zero_SCE}
       \\
       Lemma~\ref{lemma: identify_under_shift} & Identifiability of shifts in continuous random variables 
       \\
       Theorem~\ref{thm: simple_extrap_feature_map} & Adaptivity: perturbed models can inherit extrapolability 
       \\
       Theorem~\ref{thm: contraction_global_extrap} & Contraction: perturbation is a contraction on the global extrapolation uncertainty
       \\
       Theorem~\ref{thm: just_perturb} & Robustness: extrapolation uncertainty is controlled only by perturbation 
       \\
       Theorem~\ref{thm: extrap_class_charac_perturb} & Extrapolability: extrapolability can be achieved only by perturbation 
       \\
       Theorem~\ref{thm: extrap_perturb} & Some meta results on extrapolability of perturbed models 
       \\
       Theorem~\ref{thm: extrap_feature_map} & A scenario of zero perturbation cost: feature map 
       \\
       Theorem~\ref{thm: extrap_class_charac} & General version of Theorem~\ref{thm: extrap_class_charac_perturb} 
       \\
       Theorem~\ref{thm: extrap_RKHS} & A scenario of zero perturbation cost: RKHS 
       \\
       Proposition~\ref{prop: basic_extrap} & Basic properties of extrapolation uncertainty
       \\
       Proposition~\ref{prop: extrap_synonym} & Extrapolation uncertainty under perturbation: synonym replacement 
       \\
       Proposition~\ref{prop: propF1_engression} & Proposition~F.1 of \cite{ShenMeinshausen2025}
       \\
       \bottomrule
    \end{tabular}
\end{table}

To clarify the notation used throughout this appendix, we consider the version of Definition~\ref{def: dist_extrap} adapted to some general $\mathcal{X}$, rather than a fixed $\mathcal{X}$. In detail, for $\delta>0$, the extrapolation uncertainty of a class of LLMs $\mathcal{P}$ over $\mathcal{X}$ is defined exactly the same as in Definition~\ref{def: dist_extrap} but with notation $\mathcal{U}_{\mathcal{P}}(\delta)$ replaced by $\mathcal{U}_{\mathcal{P}, \mathcal{X}}(\delta)$. We say $\mathcal{P}$ is extrapolable over $\mathcal{X}$ if $\mathcal{U}_{\mathcal{P}, \mathcal{X}}(\delta) = 0$ for any $\delta>0$.
Even though all the arguments shown later are based on this adapted version of definition, they directly imply those shown in the main text, where we focus on a fixed training support $\mathcal{X}$.

We first present the following lemmas.

\medskip
\begin{lemma}[A hierarchical representation of language models]
\label{lemma: perturbation_inference}
The set $\mathcal{P}$ can be equivalently represented as
\begin{equation*}
\mathcal{P} =
\Big\{ g_\gamma(\cdot| \bm{X}_\beta), \bm{X}_\beta \sim T_\beta(\cdot| \bm{X}): \bm{X} \in \mathcal{X}^* \Big\} ,
\end{equation*}
where $\bm{X}_\beta \in \mathcal{X}^*$ denotes some random sequence of tokens, and $g_\gamma$ and $T_\beta$ are some conditional probability measures parametrized by $\gamma$ and $\beta$, respectively.
\end{lemma}

\medskip
\begin{lemma}\label{lemma: UI_imply_zero_cost}
Under Assumption~\ref{ass: UI}, $\mathcal{E}_{T,\mathcal{P}^{\textnormal{base}},\delta}$ defined in \eqref{eqn: def_SCE} is zero for any $\delta$.
\end{lemma}

\medskip
\begin{lemma}\label{lemma: identify_under_shift}
Given some continuous probability density function $f:\mathbb{R} \mapsto \mathbb{R}_{\geq 0}$ with full support, for any $a,b\in \mathbb{R}$, we have
\[
\int_{-\infty}^\infty \big| f(x-a) - f(x-b) \big| \,\mathrm{d}x = 0 
\text{\ if and only if \ }
a = b .
\]
\end{lemma}

Lemma~\ref{lemma: perturbation_inference} serves as a useful tool to explore the connection between a class of models and its version under perturbation. 
Note that using measure-theoretic notations, the lemma can be read as that any model $P_\theta\in \mathcal{P}$ can be written as
\begin{equation}\label{eqn: equiv_P_theta}
P_\theta(\cdot| \bm{X}) = \int g_\gamma(\cdot| \bm{X}_\beta) \,\mathrm{d} T_\beta(\bm{X}_\beta |\bm{X}) ,\quad \theta=(\gamma, \beta),
\end{equation}
and vice versa. 
This lemma facilitates the notation for general arguments, compared to
\[
\mathbb{E}_{\widetilde{\bm{X}}\sim T(\cdot | \bm{X})} g_\gamma(\cdot| \widetilde{\bm{X}}) = \int g_\gamma(\cdot| \widetilde{\bm{X}}) \,\mathrm{d} T(\widetilde{\bm{X}} |\bm{X}) ,
\]
where $g_\gamma$ plays the role of base model $P_\theta$ in the main text. Namely, we consider some general model which can always be represented in such hierarchical form.
Notably, $g_\gamma$ and $T_\beta$ are not identified due to this composition structure, but their roles and constraints would be more clarified in the henceforth discussion.
Throughout this appendix, we refer to $T_\beta(\cdot| \bm{X})$ the \textit{perturbation process}, and $g_\gamma(\cdot| \bm{X}_\beta)$ the \textit{inference process}.

\medskip
\begin{assumption}[General version of Assumption~\ref{ass: SP_perturb}]
\label{ass: SP}
Given $\mathcal{P}$, for any $\delta>0$ and $\bm{X}'$ with $d(\bm{X}', \mathcal{X}) \leq \delta$, there exists some perturbation process in each $P_\theta\in \mathcal{P}$ such that $T_\beta(\cdot| \bm{X}') = T_\beta(\cdot| \bm{X}_0)$ for some $\bm{X}_0\in \mathcal{X}$.
\end{assumption}

\medskip
\begin{assumption}
\label{ass: zero_SCE}
For any $\delta >0$, 
the perturbation cost $\mathcal{E}_{T,\mathcal{P}^{\textnormal{base}},\delta}$, 
defined in \eqref{eqn: def_SCE}, is zero.
\end{assumption}

Note that Assumption~\ref{ass: zero_SCE} can be fulfilled by, for instance, the settings in Theorem~\ref{thm: extrap_feature_map} or Theorem~\ref{thm: extrap_RKHS}. In both cases, Assumption~\ref{ass: zero_SCE} may be regarded as requiring the perturbation transformation $T$ to be invertible in some sense. We defer more detailed discussion to the remarks towards Theorem~\ref{thm: extrap_feature_map}.

Next, we delineate some meta results which the results presented in the main text are mainly based on.

\medskip
\begin{theorem}[Extrapolability under perturbation]\label{thm: extrap_perturb}
Let $\mathcal{P}^{\textnormal{base}} = \{ P_\theta (\cdot| \bm{X}) : \bm{X} \in \mathcal{X}^*, \theta \in \Theta^{\textnormal{base}} \}$ be any class of language models characterized by the set of parameters $\Theta^{\textnormal{base}}$. Let $T(\cdot| \bm{X})$, $\bm{X} \in \mathcal{X}^*$, be some given perturbation mechanism representing a perturbation process, and thereby define the perturbed class of language models
\begin{align*}
    &\quad\quad
    \mathcal{P}^{\textnormal{perturb}} 
    := \left\{ \int P_\theta(\cdot| \widetilde{\bm{X}}) \,\mathrm{d} T(\widetilde{\bm{X}}| \bm{X}) : \bm{X} \in \mathcal{X}^* , \theta \in \Theta^{\textnormal{base}} \right\} .
\end{align*}
Then given some set $\mathcal{X}$,
constant $\delta>0$,
the following holds.
\begin{itemize}
    \item [(i)]
    $\mathcal{U}_{
    \text{\small 
    $\mathcal{P}^{\textnormal{perturb}}, \mathcal{X}$
    }
    }(\delta)
    \leq
    \eta_T(\delta)$,
    where
    \begin{align*}
    &\quad\quad
    \eta_T(\delta) 
    :=
    \sup_{
    \substack{
    \text{\footnotesize $\bm{X}':$ } \\
    \text{\footnotesize $d(\bm{X}', \mathcal{X}) \leq \delta$ }
    }
    } 
    \inf_{
    \substack{
    \text{\footnotesize $\bm{X} \in \mathcal{X}$ } 
    }
    }
    \mathbb{TV}(T(\cdot| \bm{X}), T(\cdot| \bm{X}')) .
    \end{align*}

    \item [(ii)]
    Define
    \begin{align*}
    \rho_T(\delta) := \sup_{
    \substack{
    \text{\footnotesize $\bm{X}':$ } \\
    \text{\footnotesize $d(\bm{X}', \mathcal{X}) \leq \delta$ }
    }
    }
    \sup_{
    \substack{
    \text{\footnotesize $\widetilde{\bm{X}}:$ } \\
    \text{\footnotesize $T(\widetilde{\bm{X}}|\bm{X}') >0$ }
    }
    }
    d(\widetilde{\bm{X}}, \mathcal{X}) .
    \end{align*}

    Define the perturbation cost as
    \begin{align}
    \mathcal{E}_{T,\mathcal{P}^{\textnormal{base}},\delta}
    &:=
    \sup_{
    \substack{
    \text{\footnotesize $\bm{X}':$ } \\
    \text{\footnotesize $d(\bm{X}', \mathcal{X}) \leq \delta$ }
    }
    }
    \sup_{
    \substack{
    \text{\footnotesize $\widetilde{\bm{X}}^*:$ } \\
    \text{\footnotesize $T(\widetilde{\bm{X}}^*|\bm{X}') >0$ }
    }
    }
    \Bigg\{
    \sup_{
    \substack{
    \text{\footnotesize $P_{\gamma}= \int P_{\theta}(\cdot| \widetilde{\bm{X}}) \,\mathrm{d}T(\widetilde{\bm{X}}| \bm{X}) ,$ } \\
    \text{\footnotesize $P_{\gamma'}= \int P_{\theta'}(\cdot| \widetilde{\bm{X}}) \,\mathrm{d}T(\widetilde{\bm{X}}| \bm{X}) $ } \\
    \text{\footnotesize $\in \mathcal{P}^{\textnormal{perturb}}: $ } \\
    \text{\footnotesize $\mathbb{TV}_{\bm{X}}(P_{\gamma}, P_{\gamma'}) =0 $ } \\
    \text{\footnotesize $ \forall \bm{X} \in \mathcal{X}$ }
    }
    }
    \inf_{
    \substack{
    \text{\footnotesize $P_{\theta_0}, P_{\theta_0'} \in \mathcal{P}^{\textnormal{base}}:$ } \\
    \text{\footnotesize $\mathbb{TV}_{\bm{X}}(P_{\theta_0}, P_{\theta_0'}) =0 $ } \\
    \text{\footnotesize $ \forall \bm{X} \in \mathcal{X}$ }
    }
    } \notag \\
    &\quad\quad\quad\quad\quad\quad
    \quad\quad\quad\quad\quad\quad
    \mathbb{TV}\Big( P_{\theta}(\cdot| \widetilde{\bm{X}}^*) - P_{\theta_0}(\cdot| \widetilde{\bm{X}}^*), P_{\theta'}(\cdot| \widetilde{\bm{X}}^*) - P_{\theta_0'}(\cdot| \widetilde{\bm{X}}^*) \Big)
    \Bigg\} .
    \label{eqn: def_SCE}
    \end{align}
    
    Then
    \[
    \mathcal{U}_{
    \text{\small 
    $\mathcal{P}^{\textnormal{perturb}}, \mathcal{X}$
    }
    }(\delta)
    \leq
    \mathcal{U}_{
    \text{\small
    $\mathcal{P}^{\textnormal{base}}, \mathcal{X}$
    }
    }( \rho_T(\delta) )
    + \mathcal{E}_{T,\mathcal{P}^{\textnormal{base}},\delta} .
    \]
\end{itemize}
\end{theorem}

\medskip
\begin{theorem}[An example\footnote{See the more general version by Theorem~\ref{thm: extrap_RKHS}.}
when $\mathcal{E}_{T,\mathcal{P}^{\textnormal{base}},\delta} =0$]
\label{thm: extrap_feature_map}
Using notations in Theorem~\ref{thm: extrap_perturb}, assume each base model in $\mathcal{P}^{\textnormal{base}}$ is such that for any $i\in \mathcal{V}$, $\bm{X}\in \mathcal{X}^*$:
\[
P_\theta (i| \bm{X}) = g\left( \langle \psi_{\theta,i}, \phi(\bm{X}) \rangle \right),
\]
where $g(\cdot): \mathbb{R} \mapsto [0,1]$ is some 
affine function,
$\phi(\cdot): \mathcal{X}^* \mapsto \mathbb{R}^d$ is some latent, $d$-dimensional feature map, and $\psi_{\theta,i} \in \mathbb{R}^d$ is the vector of coefficients.

If $\{ \mathbb{E}_{\widetilde{\bm{X}} \sim T(\cdot| \bm{X})} [\phi(\widetilde{\bm{X}}) ] : \bm{X}\in \mathcal{X}\}$ spans $\mathbb{R}^{d}$, then 
for any $\delta>0$,
$\mathcal{E}_{T,\mathcal{P}^{\textnormal{base}},\delta} =0$ in part~(ii) of Theorem~\ref{thm: extrap_perturb}.
\end{theorem}

\medskip
\begin{theorem}[Extrapolability of language models]\label{thm: extrap_class_charac}
Under Assumption~\ref{ass: SP}, $\mathcal{P}$ is extrapolable.
\end{theorem}

Markedly, in Theorem~\ref{thm: extrap_perturb}, the perturbation cost defined in \eqref{eqn: def_SCE} represents the possible loss of perturbation if it undermines the structure in the base class of models.
Note that if the perturbation $T(\cdot| \bm{X})$ is Lipschitz with constant $L$ in total variation distance, then the theorem implies
$\mathcal{U}_{\text{\small $\mathcal{P}^{\textnormal{perturb}}$}}(\delta) \leq L\delta$,
thus implying extrapolability in an approximate sense that the extrapolation uncertainty decays linearly.

Note that in Theorem~\ref{thm: extrap_feature_map}, it is implicitly required that $|\mathcal{X}|\geq d$, which is very mild in the context of language generation.
The underlying idea in Theorem~\ref{thm: extrap_feature_map} is that the perturbation should not break the identifiability of models in $\mathcal{P}^{\textnormal{base}}$, which is essentially the same as in Proposition~F.1 of \citet{ShenMeinshausen2025} spelling out that certain ``pre-post-additive noise models'' are (distributionally) extrapolable. 
Notably, Theorem~\ref{thm: extrap_feature_map}'s condition on $\{ \mathbb{E}_{\widetilde{\bm{X}} \sim T(\cdot| \bm{X})} [\phi(\widetilde{\bm{X}}) ] : \bm{X}\in \mathcal{X}\}$ spanning $\mathbb{R}^{d}$ effectively requires the existence of generalized right inverse for the $d\times |\mathcal{X}|$ matrix with columns $\mathbb{E}_{\widetilde{\bm{X}} \sim T(\cdot| \bm{X})} [\phi(\widetilde{\bm{X}}) ]$, $\bm{X}\in \mathcal{X}$.

Theorem~\ref{thm: extrap_class_charac} gives a general characterization of extrapolable class of language models.
This theorem dives in a different route in characterizing distributional extrapolability, compared with Theorem~1 in \citet{ShenMeinshausen2025}. It is not merely from regression to learning some categorical distributions---the real difference is essentially about the ``complexity of learning''.
To clarify this point, consider simple inference processes that are deterministic and injective, e.g., the setup of pre-additive noise models in \citet{ShenMeinshausen2025}. Then the class of the overall data generating models is easily extrapolable as long as the perturbation process allows one to identify the model parameters, which would be straightforward if the pre-additive noise has full support, by applying Lemma~\ref{lemma: identify_under_shift}.
However, it would be impractical to assume that language generation admits such a simple form.

We further present some theoretical results on (distributional) extrapolability which can be regarded as a generalization of the results shown in the main text and of partial results by \citet{ShenMeinshausen2025}, thus of independent interest.

Some brief notation on reproducing kernel Hilbert space (RKHS). For more details, we refer to Section~4 in \citet{MantonAmblard2015}. Recall first the notation $\mathcal{X}^*$ which can essentially be any (finite or infinite) non-empty set (or simply regard it as the set of all token sequences in the context of this paper). Similarly, let $\mathcal{V}$ be in general any (finite or infinite) non-empty set (one may think of it as the set of tokens or vocabularies in the context of this paper). We denote by $\mathcal{H}_K$ the RKHS on $\mathcal{X}^*$ with some symmetric, positive (semi-)definite kernel $K(\cdot, \cdot): \mathcal{X}^* \times \mathcal{X}^* \mapsto \mathbb{R}$. Namely, $\mathcal{H}_K$ is a Hilbert space (uniquely corresponding to the specified kernel $K(\cdot, \cdot)$) of functions on $\mathcal{X}^*$ with the property that for any $\bm{X}\in \mathcal{X}^*$, $K(\bm{X}, \cdot) \in \mathcal{H}_K$ and $f(\bm{X}) = \langle f, K(\bm{X}, \cdot) \rangle_{\mathcal{H}_K}$.

\medskip
\begin{theorem}[A general scenario for $\mathcal{E}_{T,\mathcal{P}^{\textnormal{base}},\delta} =0$]\label{thm: extrap_RKHS}

Assume there exists an RKHS on $\mathcal{X}^*$, denoted by $\mathcal{H}_K$, with some measurable function $f_1,\dots, f_d: \mathcal{V} \mapsto \mathbb{R}$ such that with notations in Theorem~\ref{thm: extrap_perturb}, each base model $P_\theta \in \mathcal{P}^{\textnormal{base}}$ defines the following elements in $\mathcal{H}_K$ for $j=1,\dots,d$:
\[
h_{\theta, j}(\bm{X}) := \mathbb{E}_{i\sim P_\theta(\cdot| \bm{X})}[f_j(i)] ,
\]
with $\mathbb{E}_{i\sim P_\theta(\cdot| \bm{X})}[|f_j(i)|] <\infty$ for any $\bm{X}\in \mathcal{X}^*$. 
For each $\bm{X} \in \mathcal{X}^*$, also define the linear functional $\mathbb{L}_{\bm{X}}: \mathcal{H}_K \to \mathbb{R}$ as
\[
\mathbb{L}_{\bm{X}}(h) := \mathbb{E}_{\widetilde{\bm{X}} \sim T(\cdot| \bm{X})}[h(\widetilde{\bm{X}})] .
\]
Assume that $\mathbb{L}_{\bm{X}}(h)= \mathbb{L}_{\bm{X}}(h')$ for all $\bm{X}\in \mathcal{X}$ implies $h=h'$.
Further assume that $h_{\theta, j} = h_{\theta', j}$ for all $j=1,\dots, d$ implies $\theta= \theta'$.
Then 
for any $\delta>0$,
$\mathcal{E}_{T,\mathcal{P}^{\textnormal{base}},\delta} =0$ in part~(ii) of Theorem~\ref{thm: extrap_perturb}.

More generally, if $\mathbb{L}_{\bm{X}}(h_{\theta,j})= \mathbb{L}_{\bm{X}}(h_{\theta',j})$ for all $\bm{X}\in \mathcal{X}$ and $j=1,\dots,d$ implies $\theta=\theta'$, then we also have $\mathcal{E}_{T,\mathcal{P}^{\textnormal{base}},\delta} =0$
for any $\delta>0$.
\end{theorem}

A high-level idea by Theorem~\ref{thm: extrap_RKHS} is that base models, given some perturbation mechanism, remain uniquely identified in a feature space characterized by an RKHS.
This generalizes the underlying idea in, e.g., Appendix~F in \citet{ShenMeinshausen2025}.
As an example, we apply our Theorem~\ref{thm: extrap_perturb} and Theorem~\ref{thm: extrap_RKHS} to prove a version of their Proposition~F.1 (see Appendix~\ref{appendix: formal_theory_w_discuss}). The proposition is detailed as follows.

\medskip
\begin{proposition}[A version of Proposition~F.1 of \citet{ShenMeinshausen2025}]\label{prop: propF1_engression}
Let $\mathcal{X}$ be any set in $\mathbb{R}$ containing at least three distinct points.
With $\theta := (\theta_0, \theta_1, \theta_2)$, define the quadratic polynomial function $g_\theta(x):= \theta_0 + \theta_1 x + \theta_2 x^2$.
Then given continuous random variables $\eta, \xi$ that follow symmetric distributions with unbounded supports, and $\mathbb{E}(\eta^2)$ and $\mathbb{E}(\xi^2)$ assumed to exist, 
the class of quadratic pre-post-additive noise models
\begin{align*}
    \big\{ g_\theta (x+\eta) + \beta x +\xi: \theta\in \mathbb{R}^3, \beta\in \mathbb{R} \big\}
\end{align*}
is 
extrapolable.
\end{proposition}

In the remainder of this section, we discuss how our results can be generalized. First, it might be worth pointing out that the inference process $g_\gamma(\cdot| \bm{X}_\beta)$ is allowed to be deterministic, in which case we may write the hierarchical distribution in \eqref{eqn: equiv_P_theta} as $(g_\gamma \sharp T_\beta) (\cdot |\bm{X})$, where $g_\gamma \sharp T_\beta$ denotes the pushforward measure of $T_\beta$ by $g_\gamma$.

To reinforce our Remark~\ref{remark: extension},
further note that Theorem~\ref{thm: extrap_perturb} can hold for general distributions (i.e., not necessarily for discrete random variable), defined on any measurable set (not necessarily finite); see the footnotes in its proof. From the proof of Theorem~\ref{thm: extrap_class_charac}, it is also straightforward that Assumption~\ref{ass: SP} is always a sufficient condition for extrapolability in the general context.
Also, recall that the total variation distance can be replaced in Definition~\ref{def: dist_extrap}.
In fact, except for some quantities (involving total variation distance) that need to be adjusted based on a different distance, our results would be exactly the same.
For instance, Assumption~\ref{ass: SP} remains sufficient for extrapolability of the perturbed models, where extrapolability is defined by some general distance (compared to the specific total variation distance).



\subsection{Proofs of lemmas}

\begin{proof}[Proof of Lemma~\ref{lemma: perturbation_inference}]
This result is straightforward since $g_\gamma$ and $T_\beta$ can be arbitrary.
Setting $T_\beta$ as the identity map (i.e., Dirac delta on the input sequence) recovers the model parametrized by $\theta$, and the other direction is trivial.
\end{proof}

\begin{proof}[Proof of Lemma~\ref{lemma: UI_imply_zero_cost}]
To show Assumption~\ref{ass: UI} is sufficient for $\mathcal{E}_{T,\mathcal{P}^{\textnormal{base}},\delta} =0$, note that $\theta_0, \theta_0'$ in the definition of $\mathcal{E}_{T,\mathcal{P}^{\textnormal{base}},\delta}$ in \eqref{eqn: def_SCE} can be taken as those corresponding base models $P_\theta, P_{\theta'}$ in Assumption~\ref{ass: UI}, so that the total variation distance in \eqref{eqn: def_SCE} is zero.
This concludes the proof.
\end{proof}

\begin{proof}[Proof of Lemma~\ref{lemma: identify_under_shift}]
If $a=b$, the argument is direct. If $a\neq b$, we suppose $\int_{-\infty}^\infty | f(x-a) - f(x-b) | \,\mathrm{d}x = 0$ and will show contradiction, thereby completing the proof. Note that it must then hold that $|f(x-a) - f(x-b)| = 0$ for any $x\in\mathbb{R}$, since otherwise there always exists some interval (nearby $|f(x-a) - f(x-b)| > 0$; by the continuity of $f$) leading to $\int_{-\infty}^\infty | f(x-a) - f(x-b) | \,\mathrm{d}x >0$. This means $f(x-a) = f(x-b)$ or equivalently, $f(x)= f(x + (a-b))$ for any $x\in\mathbb{R}$, which implies $f(\cdot)$ is periodic with period $a-b$. Hence
\[
\int_{-\infty}^\infty f(x) \,\mathrm{d}x = \sum_{i=-\infty}^\infty \int_{i(a-b)}^{(i+1)(a-b)} f(x) \,\mathrm{d}x
= \sum_{i=-\infty}^\infty \int_{0}^{a-b} f(x) \,\mathrm{d}x ,
\]
but $\int_{0}^{a-b} f(x) \,\mathrm{d}x >0$ by $f(x)>0$ for all $x\in \mathbb{R}$. Hence $\sum_{i=-\infty}^\infty \int_{0}^{a-b} f(x) \,\mathrm{d}x$ diverges and contradicts with $\int_{-\infty}^\infty f(x) \,\mathrm{d}x = 1$.
\end{proof}

\subsection{Proofs of theorems}

\begin{proof}[Proof of Theorem~\ref{thm: simple_extrap_feature_map}]

By Lemma~\ref{lemma: UI_imply_zero_cost}, Assumption~\ref{ass: UI} implies Assumption~\ref{ass: zero_SCE}. Then
combining Assumption~\ref{ass: zero_SCE} and Theorem~\ref{thm: extrap_perturb}~(ii), we have
\[
    \mathcal{U}_{
    \text{\small 
    $\mathcal{P}^{\textnormal{perturb}}, \mathcal{X}$
    }
    }(\delta)
    \leq
    \mathcal{U}_{
    \text{\small
    $\mathcal{P}^{\textnormal{base}}, \mathcal{X}$
    }
    }( \rho_T(\delta) ) .
\]
As $\mathcal{P}^{\textnormal{base}}$ is extrapolable over $\mathcal{X}$, we have
$\mathcal{U}_{
    \text{\small
    $\mathcal{P}^{\textnormal{base}}, \mathcal{X}$
    }
    }(\delta) = 0$
for any $\delta>0$. Together with the above display, we arrive at
$\mathcal{U}_{
    \text{\small 
    $\mathcal{P}^{\textnormal{perturb}}, \mathcal{X}$
    }
    }(\delta) = 0$
for any $\delta>0$, thereby completing the proof.
\end{proof}

\begin{proof}[Proof of Theorem~\ref{thm: contraction_global_extrap}]
First, note that for any base model $P_\theta \in \mathcal{P}^{\textnormal{base}}$, we may write its corresponding perturbed model in $\mathcal{P}^{\text{perturb}}$ as $(TP_\theta)(\bm{X}) = \int P_\theta(\cdot| \widetilde{\bm{X}}) \, dT(\widetilde{\bm{X}}| \bm{X})$, which can be seen as the Markov operator induced by the perturbation operator $T$.

Fix any $P_{\theta}, P_{\theta'} \in \mathcal{P}^{\textnormal{base}}$ such that $\mathbb{TV}_{\bm{X}}(P_{\theta}, P_{\theta'}) =0$ for all $\bm{X} \in \mathcal{X}$.
For any $\bm{X} \in \mathcal{X}^*$, define the signed measure-valued function $h(\bm{X}):= P_\theta(\cdot| \bm{X})-P_{\theta'}(\cdot| \bm{X})$.
Then $h(\bm{X}) = 0$ for $\bm{X} \in \mathcal{X}$ by construction.
For the corresponding perturbed models, we have $(Th)(\bm{X}) = (TP_\theta)(\cdot| \bm{X}) - (TP_{\theta'})(\cdot| \bm{X})$ where $(Th)(\bm{X}) := \int h(\bm{X}') \, dT(\bm{X}'| \bm{X})$.

Then define\footnote{Note that $\|h\|_\infty = \mathcal{U}_{ \mathcal{P}^{\textnormal{base}}}^{(\infty)}$ and is the key reason why contraction only holds under global extrapolation uncertainty (unless more conditions are imposed).}
\begin{align*}
    \|h\|_\infty := 
    \sup_{\bm{X} \in \mathcal{X}^*} \mathbb{TV}_{\bm{X}} (P_{\theta}, P_{\theta'}).
\end{align*}
Fix an arbitrary $\bm{X}_0 \in \mathcal{X}$.
As $h(\bm{X}_0) =0$,
By definition, we can write $(Th)(\bm{X}_0) = \int h(\bm{X}')\,dT(\bm{X}' | \bm{X}_0) =0$.
Also, for any $\bm{X} \in \mathcal{X}^*$, we have
\begin{align*}
    (Th)(\bm{X}) - (Th)(\bm{X}_0) =
    \int h(\bm{X}')\, d\big\{ T(\bm{X}'| \bm{X})-T(\bm{X}'| \bm{X}_0) \big\} .
\end{align*}
Therefore, with $\|\cdot\|_{\mathrm{TV}}$ denoting the total variation norm, we have
\begin{align*}
    &\;\quad
    \mathbb{TV}_{\bm{X}} (TP_{\theta}, TP_{\theta'})
    =
    \big\| (Th)(\bm{X}) \big\|_{\mathrm{TV}}
    =
    \big\| (Th)(\bm{X}) - (Th)(\bm{X}_0) \big\|_{\mathrm{TV}} \\
    &\leq
    \int \|h(\bm{X}')\|_{\mathrm{TV}} \,d \big| T(\bm{X}'| \bm{X}) - T(\bm{X}'| \bm{X}_0) \big| 
    \leq
    2 \|h\|_\infty \,
    \mathbb{TV} \big( T(\cdot| \bm{X}), T(\cdot| \bm{X}_0) \big) 
    \leq
    2 \alpha(T)\, \|h\|_\infty .
\end{align*}
Taking the supremum over $\bm{X}\in \mathcal{X}^*$ gives $\sup_{\bm{X} \in \mathcal{X}^*}
\|(Th)(\bm{X}) \|_{\mathrm{TV}} \leq 2 \alpha(T)\, \|h\|_\infty$, which is equivalent to
\begin{align*}
    \sup_{\bm{X} \in \mathcal{X}^*} \mathbb{TV}_{\bm{X}}(TP_{\theta},TP_{\theta'}) \leq
    2 \alpha(T) \cdot \sup_{\bm{X} \in \mathcal{X}^*} \mathbb{TV}_{\bm{X}}(P_{\theta},P_{\theta'}) .
\end{align*}
Taking supremum over all such pairs of models $P_{\theta},P_{\theta'}$, 
and with Assumption~\ref{ass: UI},
we complete the proof of this theorem.
\end{proof}

\begin{proof}[Proof of Theorem~\ref{thm: just_perturb}]
This is direct from Theorem~\ref{thm: extrap_perturb}~(i) with $C_{T,\delta} = \eta_T(\delta)$.
\end{proof}

\begin{proof}[Proof of Theorem~\ref{thm: extrap_class_charac_perturb}]
This is direct from Theorem~\ref{thm: extrap_class_charac}.
\end{proof}

\begin{proof}[Proof of Corollary~\ref{coro:dr}]
This follows directly from the proofs of Theorems ~\ref{thm: simple_extrap_feature_map} and ~\ref{thm: extrap_class_charac_perturb}.
\end{proof}

\begin{proof}[Proof of Theorem~\ref{thm: extrap_perturb}]

Consider part~(i) first.
Fix any $\delta>0$ and $\bm{X}'$ with $d(\bm{X}', \mathcal{X}) \leq \delta$.
Recall that
\[
\eta_T(\delta) =
\sup_{
    \substack{
    \text{\footnotesize $\bm{X}': d(\bm{X}', \mathcal{X}) \leq \delta$ } 
    }
    } 
    \inf_{
    \substack{
    \text{\footnotesize $\bm{X} \in \mathcal{X}$ } 
    }
    }
    \mathbb{TV}(T(\cdot| \bm{X}), T(\cdot| \bm{X}')) .
\]
Note that $\mathcal{X}$ is a finite set since $\mathcal{X}^*$ is a finite set, thus allowing us to define a $\bm{X}^* \in \mathcal{X}$ achieving the infimum in $\eta_T(\delta)$, i.e., $\bm{X}^*$ satisfies
\footnote{Note that for all the arguments for part~(i) in Theorem~\ref{thm: extrap_perturb} to hold by allowing $|\mathcal{X}^*|$ to be infinite where no $\bm{X}^* \in \mathcal{X}$ can achieve the infimum, we simply need to update the result by replacing $\eta_T(\delta)$ with $\eta_T(\delta) +\epsilon$ for any $\epsilon>0$ and there must exists $\bm{X}^* \in \mathcal{X}$ (otherwise the infimum is at least $\eta_T(\delta) +\epsilon$) such that $\mathbb{TV}(T(\cdot| \bm{X}^*), T(\cdot| \bm{X}')) \leq \eta_T(\delta) +\epsilon$, in replacement of \eqref{eqn: proof_TV_eta}.}
\begin{equation}\label{eqn: proof_TV_eta}
\mathbb{TV}(T(\cdot| \bm{X}^*), T(\cdot| \bm{X}')) \leq \eta_T(\delta) .
\end{equation}

To ease notation, we parametrize any model in $\mathcal{P}^{\textnormal{perturb}}$ by $\gamma$ such that for any $\bm{X} \in \mathcal{X}^* , \theta \in \Theta^{\textnormal{base}}$,
\begin{equation}\label{eqn: P_perturb_reparam}
P_\gamma(\cdot | \bm{X}) := \int P_\theta(\cdot| \widetilde{\bm{X}}) \,\mathrm{d} T(\widetilde{\bm{X}}| \bm{X}) .
\end{equation}
Then consider any two perturbed models $P_\gamma(\cdot| \bm{X}), P_{\gamma'}(\cdot| \bm{X}) \in \mathcal{P}^{\textnormal{perturb}}$ (with $\theta$ and $\theta'$ denoting the parameters for the inference processes, respectively) such that for all $\bm{X} \in \mathcal{X}$, it holds that $\mathbb{TV}_{\bm{X}}(P_\gamma, P_{\gamma'}) = 0$.
In particular, $\mathbb{TV}_{\bm{X}^*}(P_\gamma, P_{\gamma'}) = 0$, implying that for any token $i$ from the vocabulary, we have
\[
0= P_\gamma(i| \bm{X}^*) - P_{\gamma'}(i| \bm{X}^*)
= \int \left\{ P_\theta(i| \widetilde{\bm{X}}) - P_{\theta'}(i| \widetilde{\bm{X}}) \right\} \,\mathrm{d} T(\widetilde{\bm{X}}| \bm{X}^*) .
\]
Based on this, on $\bm{X}'$, for any token $i$ from the vocabulary, we can write
\begin{align*}
    &\quad
    \big| P_\gamma(i| \bm{X}') - P_{\gamma'}(i| \bm{X}') \big|
    = \left| \int \left\{ P_\theta(i| \widetilde{\bm{X}}) - P_{\theta'}(i| \widetilde{\bm{X}}) \right\} \,\mathrm{d} T(\widetilde{\bm{X}}| \bm{X}') \right| \\
    &=
    \left| \int \left\{ P_\theta(i| \widetilde{\bm{X}}) - P_{\theta'}(i| \widetilde{\bm{X}}) \right\} \,\mathrm{d} T(\widetilde{\bm{X}}| \bm{X}')
    - \left\{ P_\theta(i| \widetilde{\bm{X}}) - P_{\theta'}(i| \widetilde{\bm{X}}) \right\} \,\mathrm{d} T(\widetilde{\bm{X}}| \bm{X}^*) \right| \\
    &=
    \left| \int \left\{ P_\theta(i| \widetilde{\bm{X}}) - P_{\theta'}(i| \widetilde{\bm{X}}) \right\} \,\mathrm{d} \left\{ T(\widetilde{\bm{X}}| \bm{X}') - T(\widetilde{\bm{X}}| \bm{X}^*) \right\} \right| \\
    &\leq
    \int \big| P_\theta(i| \widetilde{\bm{X}}) - P_{\theta'}(i| \widetilde{\bm{X}}) \big| \,\mathrm{d} \big| T(\widetilde{\bm{X}}| \bm{X}') - T(\widetilde{\bm{X}}| \bm{X}^*) \big| \\
    &\leq
    2 \mathbb{TV}(T(\cdot| \bm{X}^*), T(\cdot| \bm{X}')) \cdot \sup \big| P_\theta(i| \widetilde{\bm{X}}) - P_{\theta'}(i| \widetilde{\bm{X}}) \big|
    \leq
    2 \eta_T(\delta) ,
\end{align*}
where $|A|:= A^+ + A^-$ denotes the total variation measure of any measure $A$ (which can be expressed as the difference between two positive measure, i.e., $A=A^+ - A^-$); the first inequality used the fact that for any measures $A,B$:
\[
\left| \int B \,\mathrm{d}A \right|
\leq
\left| \int B \,\mathrm{d}A^+ \right| + \left| \int B \,\mathrm{d}A^- \right|
\leq
\int |B| \,\mathrm{d}A^{+} + \int |B| \,\mathrm{d}A^{-}
= \int |B| \,\mathrm{d}|A| ;
\]
the second inequality used the definition of total variation distance; and the last inequality used Equation~\eqref{eqn: proof_TV_eta} together with the fact that $\sup \big| P_\theta(i| \widetilde{\bm{X}}) - P_{\theta'}(i| \widetilde{\bm{X}}) \big|\leq 1$.
By taking supremum of the above display over $i$ in the vocabulary list, it is direct to conclude $\mathbb{TV}_{\bm{X}'}(P_{\gamma}, P_{\gamma'}) \leq  \eta_T(\delta)$ and hence
$\mathcal{U}_{
\text{\small 
$\mathcal{P}^{\textnormal{perturb}}, \mathcal{X}$
}
}(\delta)
\leq 2 \eta_T(\delta)$.
This completes the proof of part~(i).

Consider part~(ii).
Take any $\delta>0$, $\bm{X}'$ with $d(\bm{X}', \mathcal{X}) \leq \delta$, and $P_\gamma, P_{\gamma'} \in \mathcal{P}^{\textnormal{perturb}}$ with $\mathbb{TV}_{\bm{X}}(P_{\gamma}, P_{\gamma'}) = 0$ for all $\bm{X}\in \mathcal{X}$.
Then recall the notation of vocabulary list $\mathcal{V}=\{1,\dots,v\}$, we have
\footnote{For general $\mathcal{X}^*$ and $\mathcal{V}$ (possibly infinite sets), all arguments in the proof of Theorem~\ref{thm: extrap_perturb}~(ii) remain to hold, since any $\sum_{i=1}^v |\cdot|$ in the second and third lines of \eqref{eqn: perturb_base_connect_step1} can be replaced by the $L^1$ norm for a measure.}
\begin{align}
    &\quad
    \mathbb{TV}_{\bm{X}'}(P_{\gamma}, P_{\gamma'})
    = \mathbb{TV}\left( \int P_\theta(\cdot| \widetilde{\bm{X}}) \,\mathrm{d} T(\widetilde{\bm{X}}| \bm{X}') , \int P_{\theta'}(\cdot| \widetilde{\bm{X}}) \,\mathrm{d} T(\widetilde{\bm{X}}| \bm{X}') \right) \notag \\
    &=
    \frac{1}{2} \sum_{i=1}^v \left| \int P_\theta(i| \widetilde{\bm{X}}) \,\mathrm{d} T(\widetilde{\bm{X}}| \bm{X}') - \int P_{\theta'}(i| \widetilde{\bm{X}}) \,\mathrm{d} T(\widetilde{\bm{X}}| \bm{X}') \right| \notag \\
    &\leq
    \frac{1}{2} \int \sum_{i=1}^v \big| P_\theta(i| \widetilde{\bm{X}}) - P_{\theta'}(i| \widetilde{\bm{X}}) \big| \,\mathrm{d} T(\widetilde{\bm{X}}| \bm{X}') \notag \\
    &=
    \int \mathbb{TV}(P_{\theta}(\cdot| \widetilde{\bm{X}}), P_{\theta'}(\cdot| \widetilde{\bm{X}})) \,\mathrm{d} T(\widetilde{\bm{X}}| \bm{X}') \notag \\
    &\leq
    \sup_{\widetilde{\bm{X}}: T(\widetilde{\bm{X}}| \bm{X}')>0} \mathbb{TV}(P_{\theta}(\cdot| \widetilde{\bm{X}}), P_{\theta'}(\cdot| \widetilde{\bm{X}})) .
    \label{eqn: perturb_base_connect_step1}
\end{align}
Note that for any $P_{\theta_0}, P_{\theta_0'}\in \mathcal{P}^{\textnormal{base}}$, we always have
\begin{align*}
    &\quad
    \mathbb{TV}(P_{\theta}(\cdot| \widetilde{\bm{X}}), P_{\theta'}(\cdot| \widetilde{\bm{X}})) \\
    &\leq
    \mathbb{TV}(P_{\theta_0}(\cdot| \widetilde{\bm{X}}), P_{\theta_0'}(\cdot| \widetilde{\bm{X}}))
    + \mathbb{TV}(P_{\theta}(\cdot| \widetilde{\bm{X}}), P_{\theta_0}(\cdot| \widetilde{\bm{X}}))
    + \mathbb{TV}(P_{\theta'}(\cdot| \widetilde{\bm{X}}), P_{\theta_0'}(\cdot| \widetilde{\bm{X}})) ,
\end{align*}
which, if it holds that $\mathbb{TV}_{\bm{X}}(P_{\theta_0}, P_{\theta_0'}) = 0$ for all $\bm{X}\in \mathcal{X}$, then we may write \eqref{eqn: perturb_base_connect_step1} as
\begin{align*}
    &\quad
    \mathbb{TV}_{\bm{X}'}(P_{\gamma}, P_{\gamma'}) \\
    &\leq
    \sup_{
    \substack{
    \text{\footnotesize $\widetilde{\bm{X}}:$ } \\
    \text{\footnotesize $T(\widetilde{\bm{X}}|\bm{X}') >0$ }
    }
    }
    \mathbb{TV}(P_{\theta_0}(\cdot| \widetilde{\bm{X}}), P_{\theta_0'}(\cdot| \widetilde{\bm{X}})) \\
    &\quad
    + \sup_{
    \substack{
    \text{\footnotesize $\widetilde{\bm{X}}:$ } \\
    \text{\footnotesize $T(\widetilde{\bm{X}}|\bm{X}') >0$ }
    }
    } 
    \mathbb{TV}(P_{\theta}(\cdot| \widetilde{\bm{X}}), P_{\theta_0}(\cdot| \widetilde{\bm{X}})) 
    + \sup_{
    \substack{
    \text{\footnotesize $\widetilde{\bm{X}}:$ } \\
    \text{\footnotesize $T(\widetilde{\bm{X}}|\bm{X}') >0$ }
    }
    }
    \mathbb{TV}(P_{\theta'}(\cdot| \widetilde{\bm{X}}), P_{\theta_0'}(\cdot| \widetilde{\bm{X}})) \\
    &\leq
    \mathcal{U}_{
    \text{\small
    $\mathcal{P}^{\textnormal{base}}, \mathcal{X}$
    }
    }(\rho_T(\delta)) \\
    &\quad
    + \sup_{
    \substack{
    \text{\footnotesize $\widetilde{\bm{X}}:$ } \\
    \text{\footnotesize $T(\widetilde{\bm{X}}|\bm{X}') >0$ }
    }
    } 
    \mathbb{TV}(P_{\theta}(\cdot| \widetilde{\bm{X}}), P_{\theta_0}(\cdot| \widetilde{\bm{X}})) 
    + \sup_{
    \substack{
    \text{\footnotesize $\widetilde{\bm{X}}:$ } \\
    \text{\footnotesize $T(\widetilde{\bm{X}}|\bm{X}') >0$ }
    }
    }
    \mathbb{TV}(P_{\theta'}(\cdot| \widetilde{\bm{X}}), P_{\theta_0'}(\cdot| \widetilde{\bm{X}})) ,
\end{align*}
where the last inequality used the definition of the extrapolation uncertainty of $\mathcal{P}^{\textnormal{base}}$ and the definition that
\[
\rho_T(\delta) = \sup_{
    \substack{
    \text{\footnotesize $\bm{X}':$ } \\
    \text{\footnotesize $d(\bm{X}', \mathcal{X}) \leq \delta$ }
    }
    }
    \sup_{
    \substack{
    \text{\footnotesize $\widetilde{\bm{X}}:$ } \\
    \text{\footnotesize $T(\widetilde{\bm{X}}|\bm{X}') >0$ }
    }
    }
    d(\widetilde{\bm{X}}, \mathcal{X}) .
\]
With definition \eqref{eqn: def_SCE}, this concludes the proof of part~(ii) by the definition of extrapolation uncertainty, with infimum of $P_{\theta_0}, P_{\theta_0'}$ taken over $\mathcal{P}^{\textnormal{base}}$ with $\mathbb{TV}_{\bm{X}}(P_{\theta_0}, P_{\theta_0'}) = 0$ for all $\bm{X}\in \mathcal{X}$.
\end{proof}

\begin{proof}[Proof of Theorem~\ref{thm: extrap_feature_map}]

By Lemma~\ref{lemma: UI_imply_zero_cost}, it suffices to show the scenario in Theorem~\ref{thm: extrap_feature_map} fulfills Assumption~\ref{ass: UI}. 
Since $g(\cdot)$ is assumed to be 
affine,
we may consider the identity map that $g(x)=x$, without loss of generality.
To this end, fix any token $i$ from the vocabulary list. Consider any two perturbed models $P_\gamma(\cdot| \bm{X}), P_{\gamma'}(\cdot| \bm{X}) \in \mathcal{P}^{\textnormal{perturb}}$ (with $\theta$ and $\theta'$ denoting the parameters for the inference processes as in the reparameterization \eqref{eqn: P_perturb_reparam}, respectively) such that for all $\bm{X} \in \mathcal{X}$, it holds that $P_\gamma(i| \bm{X}) = P_{\gamma'}(i| \bm{X})$. Then we can write
\begin{align*}
    \sum_{\widetilde{\bm{X}} \in \mathcal{X}^*} \langle \psi_{\theta,i}, \phi(\widetilde{\bm{X}}) \rangle \cdot T(\widetilde{\bm{X}}| \bm{X})
    = 
    \sum_{\widetilde{\bm{X}} \in \mathcal{X}^*}  \langle \psi_{\theta',i}, \phi(\widetilde{\bm{X}}) \rangle  \cdot T(\widetilde{\bm{X}}| \bm{X}) ,
\end{align*}
or equivalently, by some simplification and the fact that $\sum_{\widetilde{\bm{X}} \in \mathcal{X}^*} T(\widetilde{\bm{X}}| \bm{X}) =1$, we have
\begin{align*}
    0 &= \left\langle 
    \psi_{\theta,i} - \psi_{\theta',i},
    \sum_{\widetilde{\bm{X}} \in \mathcal{X}^*} \phi(\widetilde{\bm{X}}) \cdot T(\widetilde{\bm{X}}| \bm{X})
    \right\rangle 
    =
    \Bigg\langle 
    \psi_{\theta,i} - \psi_{\theta',i},
    \mathbb{E}_{\widetilde{\bm{X}} \sim T(\cdot| \bm{X})} \left[\phi(\widetilde{\bm{X}}) \right]
    \Bigg\rangle .
\end{align*}
Then by defining the follows by concatenation (with $\bm{X}_j$ denoting the $j$-th sequence in $\mathcal{X}$, under any pre-specified ordering):
\begin{align*}
    & y := (\psi_{\theta,i} - \psi_{\theta',i})
    \in \mathbb{R}^{d}, \\
    & \Phi := \begin{pmatrix}
        \left( \mathbb{E}_{\widetilde{\bm{X}} \sim T(\cdot| \bm{X}_1)} \left[ \phi(\widetilde{\bm{X}}) \right] \right)^\top \\
         \vdots \\
        \left( \mathbb{E}_{\widetilde{\bm{X}} \sim T(\cdot| \bm{X}_{|\mathcal{X}|})} \left[ \phi(\widetilde{\bm{X}}) \right] \right)^\top
    \end{pmatrix}
    \in \mathbb{R}^{|\mathcal{X}| \times d} ,
\end{align*}
we have $\Phi \, y = 0$. Since $\big\{ \mathbb{E}_{\widetilde{\bm{X}} \sim T(\cdot| \bm{X})} \big[\phi(\widetilde{\bm{X}}) \big] : \bm{X}\in \mathcal{X} \big\}$ spans $\mathbb{R}^{d}$, the matrix $\Phi$ has full column rank, thus implying $y=0$. In other words, $\psi_{\theta,i} = \psi_{\theta',i}$, and hence the two base models $P_\theta$ and $P_{\theta'}$ are the same. This completes the proof.
\end{proof}

\begin{proof}[Proof of Theorem~\ref{thm: extrap_class_charac}]
Fix any $\delta >0$, and consider any $\bm{X}'$ with $d(\bm{X}', \mathcal{X}) \leq \delta$. Then consider any two models $P_\theta, P_{\theta'} \in \mathcal{P}$ such that $\mathbb{TV}_{\bm{X}}(P_{\theta}, P_{\theta'}) =0$ for all $\bm{X} \in \mathcal{X}$. In other words, $P_\theta(\cdot |\bm{X}) = P_{\theta'}(\cdot |\bm{X})$ for all $\bm{X} \in \mathcal{X}$.
Denote by $\theta =(\gamma, \beta)$ and $\theta' =(\gamma', \beta')$ using \eqref{eqn: equiv_P_theta}.
By Assumption~\ref{ass: SP}, for $P_\theta$ and $P_{\theta'}$, there are some perturbation processes and $\bm{X}_0\in \mathcal{X}$ such that $T_\beta(\cdot| \bm{X}') = T_\beta(\cdot| \bm{X}_0)$ and $T_{\beta'}(\cdot| \bm{X}') = T_{\beta'}(\cdot| \bm{X}_0)$.
Thus, we can fix the representation from \eqref{eqn: equiv_P_theta} under the perturbation processes from Assumption~\ref{ass: SP}.
Then for $\bm{X}'$, we may write
\begin{align*}
    & P_\theta(\cdot |\bm{X}') = \int g_\gamma(\cdot| \bm{X}_\beta) \,\mathrm{d} T_\beta(\bm{X}_\beta |\bm{X}') 
    = \int g_\gamma(\cdot| \bm{X}_\beta) \,\mathrm{d} T_\beta(\bm{X}_\beta |\bm{X}_0) , \\
    & P_{\theta'}(\cdot |\bm{X}') = \int g_{\gamma'}(\cdot| \bm{X}_\beta) \,\mathrm{d} T_{\beta'}(\bm{X}_\beta |\bm{X}')
    = \int g_{\gamma'}(\cdot| \bm{X}_\beta) \,\mathrm{d} T_{\beta'}(\bm{X}_\beta |\bm{X}_0).
\end{align*}
However, the RHS of the two lines above must equal, by recalling that $P_\theta$ and $P_{\theta'}$ agree on all $\bm{X} \in \mathcal{X}$, i.e.,
\begin{align*}
    \int g_\gamma(\cdot| \bm{X}_\beta) \,\mathrm{d} T_\beta(\bm{X}_\beta |\bm{X})
    = P_\theta(\cdot |\bm{X})
    = P_{\theta'}(\cdot |\bm{X}) 
    = \int g_{\gamma'}(\cdot| \bm{X}_\beta) \,\mathrm{d} T_{\beta'}(\bm{X}_\beta |\bm{X}) .
\end{align*}
This shows $P_\theta(\cdot |\bm{X}')= P_{\theta'}(\cdot |\bm{X}')$, thereby implying $\mathbb{TV}_{\bm{X}'}(P_{\theta}, P_{\theta'}) =0$ which holds true uniformly over all $\bm{X}'$ in the neighboring sets in $\mathcal{X}$ (detailed at the beginning). This concludes the extrapolability of $\mathcal{P}$ over $\mathcal{X}$ and hence the sufficiency of Assumption~\ref{ass: SP}.
\end{proof}

\begin{proof}[Proof of Theorem~\ref{thm: extrap_RKHS}]
We generalize the proof of Theorem~\ref{thm: extrap_feature_map} and in particular, we will verify Assumption~\ref{ass: UI} therein which has shown to be sufficient to $\mathcal{E}_{T,\mathcal{P}^{\textnormal{base}},\delta} =0$.

Consider any two perturbed models $P_\gamma(\cdot| \bm{X}), P_{\gamma'}(\cdot| \bm{X}) \in \mathcal{P}^{\textnormal{perturb}}$ (with $\theta$ and $\theta'$ denoting the parameters for the inference processes as in the reparameterization \eqref{eqn: P_perturb_reparam}, respectively) such that for all $\bm{X} \in \mathcal{X}$, it holds that $P_\gamma(\cdot| \bm{X}) = P_{\gamma'}(\cdot| \bm{X})$.
Hence for any $j=1,\dots,d$, $\bm{X} \in \mathcal{X}$, by Fubini's theorem and the fact that $f_j(i)$ is integrable w.r.t.\ any base model, we have
\begin{align*}
    &\quad
    \mathbb{L}_{\bm{X}}(h_{\theta, j})
    =
    \int_{\widetilde{\bm{X}} \in \mathcal{X}^*} h_{\theta, j}(\widetilde{\bm{X}}) \,\mathrm{d} T(\widetilde{\bm{X}} | \bm{X}) \\
    &=
    \int_{\widetilde{\bm{X}} \in \mathcal{X}^*} \mathbb{E}_{i\sim P_\theta(\cdot| \widetilde{\bm{X}})}[f_j(i)] \,\mathrm{d} T(\widetilde{\bm{X}} | \bm{X})
    =
    \int_{\widetilde{\bm{X}} \in \mathcal{X}^*} \int_{i\in \mathcal{V}} f_j(i) \,\mathrm{d} P_\theta(i| \widetilde{\bm{X}}) \,\mathrm{d} T(\widetilde{\bm{X}} | \bm{X}) \\
    &=
    \int_{i\in \mathcal{V}} f_j(i) \,\mathrm{d} P_\gamma(i| \bm{X})
    = \int_{i\in \mathcal{V}} f_j(i) \,\mathrm{d} P_{\gamma'}(i| \bm{X}) 
    =
    \int_{\widetilde{\bm{X}} \in \mathcal{X}^*} \int_{i\in \mathcal{V}} f_j(i) \,\mathrm{d} P_{\theta'}(i| \widetilde{\bm{X}}) \,\mathrm{d} T(\widetilde{\bm{X}} | \bm{X}) \\
    &=
    \int_{\widetilde{\bm{X}} \in \mathcal{X}^*} \mathbb{E}_{i\sim P_{\theta'}(\cdot| \widetilde{\bm{X}})}[f_j(i)] \,\mathrm{d} T(\widetilde{\bm{X}} | \bm{X}) 
    =
    \int_{\widetilde{\bm{X}} \in \mathcal{X}^*} h_{\theta', j}(\widetilde{\bm{X}}) \,\mathrm{d} T(\widetilde{\bm{X}} | \bm{X}) 
    = \mathbb{L}_{\bm{X}}(h_{\theta', j}) .
\end{align*}
Since the above holds for all $\bm{X}\in \mathcal{X}$, by our assumption we obtain $h_{\theta, j} = h_{\theta', j}$.
Similarly, this holds for all $j=1,\dots, d$, and by our assumption we conclude $\theta = \theta'$.
The more general assumption also leads to this implication (more directly).
This shows Assumption~\ref{ass: UI} and hence completes the proof.
\end{proof}

\subsection{Proofs of propositions}

\begin{proof}[Proof of Proposition~\ref{prop: basic_extrap}]
Part~(i) is direct by the definitions.
Similarly for part~(ii) by noting that the set $\{\bm{X}':d(\bm{X}',\mathcal{X})\leq \delta_1\}$ is contained in $\{\bm{X}':d(\bm{X}',\mathcal{X})\leq \delta_2\}$.

For part~(iii),
note that since $\mathcal{X}^*$ is finite, the set
\begin{align*}
    \mathcal{D}_{\mathcal{X}^*} := \left\{ d(\bm{X}, \mathcal{X}): \bm{X} \in \mathcal{X}^* \right\}
\end{align*}
is a finite subset of $[0, \infty)$.
Hence for every $\delta_0 \geq 0$, there exists $\varepsilon_0 > 0$ such that $(\delta_0, \delta_0+ \varepsilon_0) \cap \mathcal{D}_{\mathcal{X}^*} = \emptyset$. As $\mathcal{X} \subseteq \mathcal{X}^*$, we thus have for any $\delta \in [\delta_0, \delta_0+ \varepsilon_0)$ that
\begin{align*}
    \left\{ \bm{X}: d(\bm{X}, \mathcal{X}) \leq \delta \right\}
    =
    \left\{ \bm{X}: d(\bm{X}, \mathcal{X}) \leq \delta_0 \right\} ,
\end{align*}
so that $\mathcal{U}_{\mathcal{P}}(\delta) = \mathcal{U}_{\mathcal{P}}(\delta_0)$ on this interval. Consequently, we have $\lim_{\delta \downarrow \delta_0} \mathcal{U}_{\mathcal{P}}(\delta) = \mathcal{U}_{\mathcal{P}}(\delta_0)$, which concludes right-continuity.

Finally, consider part~(iv).
By Definition~\ref{def: dist_extrap}, $\mathcal{U}_{\mathcal{P}}(\delta)$ is upper bounded by $1$. Together with parts~(i) and (ii) of this proposition, the limit exists due to the monotone convergence theorem.
This completes the proof of this proposition.
\end{proof}

\begin{proof}[Proof of Proposition~\ref{prop: extrap_synonym}]
First, we clarify in the statement of this theorem that the notation $M(\cdot, \cdot)$ is the Hamming distance.
By Theorem~\ref{thm: extrap_perturb}, it suffices to show $T(\cdot| \bm{X})$ is Lipschitz with constant $(1-\beta)$ in total variation distance (w.r.t.\ the Hamming distance). That is, we want $\mathbb{TV}(T(\cdot| \bm{X}), T(\cdot| \bm{X}')) \leq (1-\beta) M(\bm{X}, \bm{X}')$.
If this holds, we also have $\alpha(T) \leq (1-\beta) \ell$ since $M(\cdot, \cdot)\leq \ell$ by construction.

We first construct a coupling (cf.\ Definition~4.1.1 of \citet{Roch2024}) between $T(\cdot| \bm{X})$ and $T(\cdot| \bm{X}')$. To do this, for position $l=1,\dots, L$, independently generate $U_l \sim \text{Bernoulli}(\beta)$ and some uniformly random draw from 
$Q_0$, 
denoted as $\widetilde{X}_l$. Then construct $\bm{Y} =(Y_1,\dots, Y_L)$ and $\bm{Y}' =(Y_1',\dots, Y_L')$ as token sequences such that for any $l=1,\dots, L$,
\[
Y_l = 
\begin{cases}
    X_l & \text{if $U_l=0$},\\
    \widetilde{X}_l & \text{if $U_l=1$}.
\end{cases} ,\quad
Y_l' = 
\begin{cases}
    X_l' & \text{if $U_l=0$},\\
    \widetilde{X}_l & \text{if $U_l=1$}.
\end{cases}
\]
By our perturbation procedure, $\bm{Y}\sim T(\cdot| \bm{X})$ and $\bm{Y}' \sim T(\cdot| \bm{X}')$, so that $(\bm{Y}, \bm{Y}')$ is a coupling of the measures of $\bm{X}$ and $\bm{X}'$.
Then by the coupling inequality (cf.\ Lemma~4.1.11 of \citet{Roch2024}), we have
\begin{equation}
\label{eqn: coupling_inequality}
\mathbb{TV}(T(\cdot| \bm{X}), T(\cdot| \bm{X}')) \leq \mathbb{P}(\bm{Y} \neq \bm{Y}') .
\end{equation}

Let $Z\subseteq \{1,\dots, L\}$ be the set of locations where $\bm{X}$ and $\bm{X}'$ differ.
Then observe that $\bm{Y}$ can only equal $\bm{Y}'$ if $U_l=1$ for all $l\in Z$, and hence by the independence over $U_l$, we have
\begin{align*}
    &\quad
    \mathbb{P}(\bm{Y} \neq \bm{Y}') = 1 - \mathbb{P}(\bm{Y} = \bm{Y}') = 1 - \prod_{l\in Z} \mathbb{P}(U_l = 1) = 1 - \beta^{|Z|} \\
    &=
    1 - [1- (1-\beta)]^{|Z|}
    \leq
    1 - \big[ 1 - (1-\beta) \cdot |Z| \big]
    = (1-\beta) \cdot |Z|
    = (1-\beta) \cdot M(\bm{X}, \bm{X}')
\end{align*}
where the last inequality used the Bernoulli's inequality, and the last equality used the definition of Hamming distance.
Combined with \eqref{eqn: coupling_inequality}, we may conclude the proof for $\alpha(T)$.
By the above arguments, the result for $C_{T,\delta}$ also follows directly, using its construction in the proof of Theorem~\ref{thm: just_perturb}.
Lastly, note that when $Q_0$ has support outside of $\mathcal{X}$, the only scenario for $\mathbb{TV}_{\bm{X}}(P_\gamma, P_{\gamma'}) =0$ (using notations in Assumption~\ref{ass: UI}) is when $\mathbb{TV}_{\bm{X}}( (1-\beta) P_\theta, (1-\beta) P_{\theta'}) =0$ and thus $\mathbb{TV}_{\bm{X}}(P_\theta, P_{\theta'}) =0$.
This completes this proof.
\end{proof}

\begin{proof}[Proof of Proposition~\ref{prop: propF1_engression}]
First, note that now $\mathcal{X}^* = \mathbb{R}$ and $\mathcal{V}=\mathbb{R}$. We define
\[
g_{\gamma}(x) := g_\theta(x) + \beta x
= \theta_0 + (\theta_1 +\beta) x + \theta_2 x^2 ,
\]
with $\gamma=(\theta_0, \theta_1 +\beta, \theta_2)$. Note that it suffices to identify the sum $\theta_1 +\beta$ rather than the individual $\theta_1$ and $\beta$, since the former already identifies the function $g_{\gamma}(\cdot)$.
We will read each quadratic pre-post-additive noise model as in a base model $P_\theta(\cdot| x)$ being the distribution of $g_{\gamma}(x) + \xi$, and a perturbation $T(\cdot| x)$ being the distribution of $x+\eta$. 

Consider the base model first. Denoting the probability density function of $\xi$ as $f(\cdot)$, if $P_\theta(\cdot| x_0) = P_{\theta'}(\cdot| x_0)$ for any $x_0$, we have
\[
\int_{-\infty}^\infty \big| f(x - g_{\gamma}(x_0)) - f(x - g_{\gamma'}(x_0)) \big| \,\mathrm{d}x = 0 ,
\]
which, by Lemma~\ref{lemma: identify_under_shift}, implies $g_{\gamma}(x_0) = g_{\gamma'}(x_0)$. As both $g_{\gamma}(\cdot)$ and $g_{\gamma'}(\cdot)$ are quadratic, $\mathcal{X}$ containing at least three distinct points guarantee that $\theta_0 = \theta_0'$, $\theta_1 + \beta_1 = \theta_1' + \beta_1'$, and $\theta_2 = \theta_2'$ (cf.\ the fact that the corresponding Vandermonde matrix is invertible). Note that albeit $\theta_1, \beta_1$ cannot be identified, it suffices to identify the function $g_{\gamma}$ since all coefficients in the polynomial function are identified. That is, we now have $g_{\gamma}(x) = g_{\gamma'}(x)$ for any $x\in\mathbb{R}$. Hence the class of base models is extrapolable over $\mathcal{X}$, i.e.,
$\mathcal{U}_{
    \text{\small
    $\mathcal{P}^{\textnormal{base}}, \mathcal{X}$
    }
    } (\delta) = 0$ for any $\delta >0$.

By Theorem~\ref{thm: extrap_perturb}, it remains to show $\mathcal{E}_{T,\mathcal{P}^{\textnormal{base}},\delta} =0$, for which we will use Theorem~\ref{thm: extrap_RKHS} with kernel $K(x, y) = 1+ xy + x^2 y^2$, so that this RKHS consists of all polynomial functions with maximum degree 2. Further let $f_1(x) = x$ and
\[
h_{\theta,1}(x) := \mathbb{E}_{i\sim P_{\theta}(\cdot| x)} [f_1(i)]
= \mathbb{E}_{i\sim P_{\theta}(\cdot| x)} [i]
= g_\gamma(x) ,
\]
where the last equality used the fact that the distribution of $\xi$ is symmetric. 
We may also verify that $\mathbb{E}_{i\sim P_{\theta}(\cdot| x)} [|f_1(i)|] = \mathbb{E}_{i\sim P_{\theta}(\cdot| x)} [|i|] <\infty$ by $\mathbb{E}(\xi^2)< \infty$.
Furthermore, define the linear functional
\begin{align*}
    &\quad
    \mathbb{L}_{x}(h_{\theta,1}) := \mathbb{E}_{y\sim T(\cdot| x)} [h_{\theta,1}(y)]
    = \mathbb{E}_{y\sim T(\cdot| x)} [g_\gamma(y)] \\
    &=
    \mathbb{E}_{y\sim T(\cdot| x)} \left[ 
    \theta_0 + (\theta_1 +\beta) y + \theta_2 y^2
    \right]
    =
    \theta_0 + (\theta_1 +\beta) x + \theta_2 \mathbb{E}(\eta^2) + \theta_2 x^2 ,
\end{align*}
by the fact that the distribution of $\eta$ is symmetric. Then by Theorem~\ref{thm: extrap_RKHS}, it suffices to show $\theta_0 + (\theta_1 +\beta) x + \theta_2 \mathbb{E}(\eta^2) + \theta_2 x^2 = \theta_0' + (\theta_1' +\beta') x + \theta_2' \mathbb{E}(\eta^2) + \theta_2' x^2$ for sufficiently many $x\in\mathbb{R}$ (to be detailed soon) implies $(\theta_0, \theta_1 +\beta, \theta_2) = (\theta_0', \theta_1' +\beta', \theta_2')$. Note that we may rewrite
\[
0 = \Big[ (\theta_0 -\theta_0') + (\theta_2 - \theta_2') \mathbb{E}(\eta^2) \big] + \big[ (\theta_1 +\beta) - (\theta_1' +\beta') \big] x  + (\theta_2 - \theta_2') x^2 ,
\]
which, if it holds for at least three distinct $x$, gives us the desired. This completes the verification of conditions in Theorem~\ref{thm: extrap_RKHS} and hence shows Proposition~F.1 of \citet{ShenMeinshausen2025}.
\end{proof}

\section{Details of experiments}\label{appendix: experiments}

\subsection{Implementation details in the training phase}
Perturbation is not applied to every token, as doing so may generate a completely random prefix. Instead, we perturb an $\alpha$ proportion of tokens in the prefix, where $\alpha$ denotes the perturbation intensity. Moreover, to be compatible with the existing training pipeline so as to simplify implementation, we perturb the original dataset $\mathcal{D}$ to obtain a collection of new datasets $\widetilde{\mathcal{D}}^{(1)}, \ldots, \widetilde{\mathcal{D}}^{(m)}$. Finally, the combined dataset $\mathcal{D} \cup \widetilde{\mathcal{D}}^{(1)} \cup \cdots \cup \widetilde{\mathcal{D}}^{(m)}$ is used to train $P_{\theta}$ based on some specified scoring rule. We fix $m = 1$ throughout the experiments. 

Notably, our method makes use of combined datasets. Hence to ensure a fair comparison, we combine the original dataset and its copy when conducting the training procedure using the classical approach (i.e., without perturbation). The model parameters are then updated after each epoch.

\subsection{Details on simulation studies}\label{appendix:simulation}

\textbf{Setting of Algorithm~\ref{alg: pretrain}.} We set $m = 2$ in Algorithm~\ref{alg: pretrain}. Specifically, when $j = 1$, we set $\widetilde{\bm{X}}^{(i,j)}_{t} = \bm{X}^{(i)}_{t}$. That is, the first half of the perturbed dataset is identical to the original dataset. For $j = 2$, we merely consider tokens that belong to the following set:
\begin{align*}
\left\{ v \in \mathcal{V} \mid 
P\big(v | \bm{X}^{(i)}_{t-1}\big) > 2 / |\mathcal{V}|
\quad \text{and} \quad
P\big(\bm{X}^{(i)}_{t+1} | v \big) > 2/|\mathcal{V}|
\right\},
\end{align*}
which is designed to mimic the set of synonyms. Each token $v$ in the above set is assigned with a probability $c P\big(v | \bm{X}^{(i)}_{t-1}\big)$ where $c$ is a normalizing constant such that all probabilities are summed to one. 

\textbf{Training configuration.} We train the model using the Adam optimizer \citep{kingma2015adam} with a learning rate of $10^{-3}$ and a weight decay of $10^{-4}$. The batch size is set to 500, and the model is trained for 25 epochs. In particular, for the classical method, since it does not involve perturbation, we simply replicate the dataset once and combine it with the original dataset to match the size of the training data used by our method. This ensures that both our method and the classical method follow the same and fair training procedure.

\end{document}